\documentclass[11pt,a4paper,english]{article}
\usepackage{cmap}
\usepackage[T1]{fontenc}
\usepackage[utf8x]{inputenc}
\usepackage{amsthm}
\usepackage{amsmath}
\usepackage{amssymb}
\usepackage{graphicx}
\usepackage{esint}

\usepackage{microtype}

\usepackage{ifthen}
\usepackage{xcolor}
\usepackage{hyperref}
\hypersetup{
  colorlinks   = true, 
  urlcolor     = red, 
  linkcolor    = blue, 
  citecolor   = green 
}

\newcommand{\markupdraft}[2]{
    \ifthenelse{\equal{#1}{display}}{#2}{}
    \ifthenelse{\equal{#1}{color}}{\color{#2}}{}
}

\newcommand{\newcolored}[3][]{{\markupdraft{color}{#2}#3}
    \ifthenelse{\equal{#1}{}}{}{\markupdraft{display}{{\color{yellow!70!black}[#1]}}}}

\newcommand{\del}[2][]{{\markupdraft{display}{{\color{orange}[removed: "#2"[#1]]}}}} 

\renewcommand{\markupdraft}[2]{}  
\newcommand{\NDY}[1]{\markupdraft{display}{{\color[rgb]{0.0,0.8,0.4}[NdYann: #1]}}}
\newcommand{\NDL}[1]{\markupdraft{display}{{\color[rgb]{0.5,0.2,0.2}[NdLudo: #1]}}}

\makeatletter

\newcommand{\noun}[1]{\textsc{#1}}
\providecommand{\tabularnewline}{\\}

  \theoremstyle{plain}
  \newtheorem{thm}{\protect\theoremname}
  \newtheorem{defi}[thm]{\protect\definitionname}
  \newtheorem{prop}[thm]{\protect\propositionname}
  
  \newtheorem{rem}[thm]{\protect\remarkname}

\makeatother

\usepackage{babel}
\providecommand{\propositionname}{Proposition}
\providecommand{\theoremname}{Theorem}
\providecommand{\corollaryname}{Corollary}
\providecommand{\definitionname}{Definition}
\providecommand{\remarkname}{Remark}

\DeclareMathOperator*{\argmax}{arg\,max}
\DeclareMathOperator*{\argmin}{arg\,min}
\newcommand{\x}{\mathbf{x}}
\newcommand{\h}{\mathbf{h}}
\newcommand{\D}{\mathcal{D}}
\newcommand{\E}{\mathbb{E}}
\newcommand{\deq}{\mathrel{\mathop:}=}

\newcommand{\BLM}[2]{\hat{Q}_{#1,#2}}

\newcommand{\U}{\mathcal{U}}
\newcommand{\UIM}{\mathcal{U}^\mathrm{model}}
\newcommand{\cond}{^{\mathrm{cond}}}
\newcommand{\KL}[2]{D_\mathrm{KL}(#1 \!\parallel\! #2)} 
\DeclareMathOperator{\sigm}{sigm}

\renewcommand{\geq}{\geqslant}
\renewcommand{\leq}{\leqslant}

\title{Layer-wise training of deep generative models}

\author{Ludovic Arnold, Yann Ollivier}
\date{}

\begin{document}

\maketitle

\begin{abstract}
When using deep, multi-layered architectures to build generative models of data, it is difficult to train all layers at
once. We propose a layer-wise training procedure admitting a performance
guarantee compared to the global optimum. It is based on an optimistic
proxy of future performance, the \emph{best latent marginal}. We
interpret auto-encoders in this setting as generative models, by showing that they train a lower
bound of this criterion. We test the new learning procedure against
a
state of the art method (stacked RBMs),
and find it to improve performance. Both
theory and experiments highlight the importance, when training deep
architectures, of using an inference
model (from data to hidden variables) richer than the generative model
(from hidden variables to data).
\end{abstract}

\NDY{Did you recheck the proofs I wrote? Mistakes are easy...}
\NDL{I'll give you my proof comments bellow with the tag "proofchecking"}

\section*{Introduction}

Deep architectures, such as multiple-layer neural
networks, have recently been the object of a lot of interest and
have been shown to provide state-of-the-art performance on many problems \cite{Bengio2012}. A key aspect of deep learning is to help in learning
better representations of the data, thus reducing the need for
hand-crafted features, a very time-consuming process requiring expert knowledge.

Due to the difficulty of training a whole deep network at once, a so-called \emph{layer-wise} procedure is used as an approximation \cite{Hinton2006,Bengio2007a}. However,
a long-standing issue is the justification of this
layer-wise training: although
the method has shown its merits in practice, theoretical justifications
fall somewhat short of expectations. A
frequently cited result
\cite{Hinton2006} is a proof that adding layers increases a
so-called variational
lower bound on the log-likelihood of the model, and therefore that adding
layers \emph{can} improve performance.

We reflect on the validity of layer-wise training
procedures, and discuss in what way and with what assumptions they can be
construed as being equivalent to the non-layer-wise, that is,
whole-network, training. This leads us to a new approach for training
deep generative models, 
using a new criterion for optimizing each layer starting from the bottom
and for transferring the problem upwards to the next layer. Under the right conditions, this new
layer-wise approach is equivalent to optimizing the log-likelihood of
the full deep generative model (Theorem~\ref{thm:main}).

As a first step, in Section~\ref{sec:dgm} we re-introduce the general form of deep generative
models,
and derive the
gradient of the log-likelihood for deep models. This gradient is seldom ever
considered because it is considered intractable and requires sampling
from
complex distributions. Hence the need for a simpler, layer-wise training
procedure.

We then show (Section~\ref{sec:mainthm}) how an
optimistic criterion, the \emph{BLM upper bound}, can be used to train
\emph{optimal} lower layers provided subsequent training of upper layers is
successful, and discuss what criterion to use to
transfer the learning problem to the upper layers.

This leads to a discussion of the relation of this procedure with
stacked restricted Boltzmann machines (SRBMs)
and auto-encoders (Sections~\ref{sec:SRBMs}
and~\ref{sec:AA}),
in which a new justification is found for auto-encoders as optimizing the
lower part of a deep generative model.

In Section~\ref{sec:whyqd} we spell out the theoretical advantages of
using a model for the hidden variable $\h$ having the form
$Q(\h)=q(\h|\x)P_\mathrm{data}(\x)$ when looking for hidden-variable generative models
of the data $\x$, a scheme close to that of auto-encoders.

Finally, we discuss new applications and perform experiments
(Section~\ref{sec:exp}) to validate the approach and compare it to
state-of-the-art methods, on two new deep datasets, one synthetic and one
real. In particular we introduce auto-encoders with rich inference
(AERIes) which are auto-encoders modified according to this framework. 

Indeed both theory and experiments strongly suggest that,
when using stacked auto-associators or similar deep architectures, the
inference part (from data to latent variables) should use a much
richer model than the generative part (from latent variables to data), in
fact, as rich as possible.
Using richer inference helps to find much better parameters for the
\emph{same} given generative model.

\section{Deep generative models\label{sec:dgm}}

Let us go back to the basic formulation of training a deep
architecture as a traditional learning problem: optimizing the parameters
of the whole architecture seen as a probabilistic generative model of the
data.

\subsection{\label{sub:p-decomposition}Deep models: probability decomposition}

The goal of generative learning is to estimate the parameters
$\theta=(\theta_{1},\dots,\theta_{n})$
of a distribution $P_{\theta}(\mathbf{x})$ in order to approximate
a data distribution $P_{\mathcal{D}}(\mathbf{x})$ on some observed
variable $\mathbf{x}$. 


The recent development of deep architectures \cite{Hinton2006,Bengio2007a} has given importance
to a particular case of \emph{latent variable models} in which the 
distribution of $\x$ can be decomposed as a sum over states of latent
variables $\h$,
\[
P_\theta(\x)=\sum_\h P_{\theta_{1},\dots,\theta_{k}}(\x|\h)
P_{\theta_{k+1},\dots,\theta_{n}}(\mathbf{h})
\]
with separate parameters for the marginal probability of
$\mathbf{h}$ and the conditional probability of $\mathbf{x}$ given
$\mathbf{h}$.
Setting $I=\{1,2,\dots,k\}$ such that $\theta_{I}$
is the set of parameters of $P(\mathbf{x}|\mathbf{h})$ and $J=\{k+1,\dots,n\}$
such that $\theta_{J}$ is the set of parameters of $P(\mathbf{h})$, this
rewrites as
\begin{eqnarray}
P_{\theta}(\mathbf{x}) & = & \sum_{\mathbf{h}}P_{\theta_{I}}(\mathbf{x}|\mathbf{h})P_{\theta_{J}}(\mathbf{h})\label{eq:p-decomposition}
\end{eqnarray}

In deep architectures, the same kind of decomposition is
applied to $\h$ itself recursively, thus defining a layered model with
several hidden layers 
$\mathbf{h}^{(1)}$, $\mathbf{h}^{(2)}$, \dots, $\mathbf{h}^{(k_\mathrm{max})}$,
namely
\begin{eqnarray}
P_{\theta}(\mathbf{x}) & = &
\sum_{\mathbf{h}^{(1)}}P_{\theta_{I_{0}}}(\mathbf{x}|\mathbf{h}^{(1)})P_{\theta_{J_0}}(\mathbf{h}^{(1)})\\
P(\mathbf{h}^{(k)}) & = &
\sum_{\mathbf{h}^{(k+1)}}P_{\theta_{I_{k}}}(\mathbf{h}^{(k)}|\mathbf{h}^{(k+1)})P_{\theta_{J_{k}}}(\mathbf{h}^{(k+1)}),\,\,\,1\leq
k\leq k_\mathrm{max}-1
\label{eq:deepmodel}
\end{eqnarray}

At any one time, 
we will only be interested in one step of this
decomposition. Thus for simplicity, 
we consider that the distribution of interest is on the
observed variable $\mathbf{x}$, with latent variable $\mathbf{h}$. The
results extend to the other layers of the decomposition by renaming
variables. 

In Sections~\ref{sec:SRBMs} and~\ref{sec:AA} we quickly present two
frequently used deep architectures, stacked
RBMs and auto-encoders, within this framework.

\subsection{Data log-likelihood}

The goal of the learning procedure, for a probabilistic generative
model, is generally to maximize the log-likelihood of the data under the
model, namely, to find the value of the parameter
$\theta^{*}=(\theta_{I}^{*},\theta_{J}^{*})$ achieving
\begin{eqnarray}
\theta^{*} & \deq & \argmax_{\theta}\mathbb{E}_{\mathbf{x}\sim
P_{\mathcal{D}}}\left[\log P_{\theta}(\mathbf{x})\right]
\label{eq:objective}
\\
& = & \argmin_{\theta}\KL{P_\D}{P_\theta},
\label{eq:KLobjective}
\end{eqnarray}
where $P_{\mathcal{D}}$ is the empirical data distribution, and $\KL{\cdot}{\cdot}$ is
the Kullback--Leibler divergence. (For simplicity we assume this
optimum is unique.)

An obvious way to tackle this problem would be a gradient ascent over the
full parameter $\theta$. However, this is impractical for deep
architectures (Section~\ref{sec:grad} below).

It would be easier to be able to train deep architectures in a layer-wise
fashion, by first training the parameters $\theta_I$ of the bottom layer,
deriving a new target distribution for the latent variables $\h$, and then
training $\theta_J$ to reproduce this target distribution on $\h$,
recursively over the layers, till one reaches the top layer on which,
hopefully, a simple probabilistic generative model can be used.

Indeed this is often done in practice, except that the objective
\eqref{eq:objective} is replaced with a surrogate objective. For
instance, for architectures made of stacked RBMs, at each level the
likelihood of a single  RBM is maximized, ignoring the fact that it is to
be used as part of a deep architecture, and moreover often using a
further approximation to the likelihood such as contrastive divergence
\cite{Hinton2002}. Under specific conditions (i.e., initializing the upper
layer with an upside-down version of the current RBM), it can be shown
that adding a layer improves a lower bound on performance
\cite{Hinton2006}.

We address in Section~\ref{sec:Layer-wise-deep-learning} the
following questions: Is it possible to compute or estimate the optimal
value of the parameters $\theta^\ast_I$ of the bottom layer, without
training the whole model? Is it possible to compare two values of
$\theta_I$ without training the whole model? The latter would be
particularly convenient
for hyper-parameter selection, as it would allow to compare lower-layer models
before the upper layers are trained, thus significantly reducing the size
of the hyper-parameter search space from exponential to linear in the
number of layers.

We propose a procedure aimed at reaching the global optimum $\theta^\ast$
in a layer-wise fashion, based on an optimistic estimate of
log-likelihood, the \emph{best latent marginal (BLM) upper bound}. We
study its theoretical guarantees in
Section~\ref{sec:Layer-wise-deep-learning}. In Section~\ref{sec:exp} we
make an experimental comparison between stacked RBMs,
auto-encoders modified according to this scheme, and vanilla
auto-encoders, on two simple but deep datasets.

\subsection{Learning by gradient ascent for deep
architectures}\label{sec:grad}

Maximizing the likelihood of
the data distribution $P_{\mathcal{D}}(\mathbf{x})$ under a model,
or equivalently minimizing the KL-divergence
$\KL{P_\D}{P_\theta}$,
is usually done with gradient ascent in the parameter space.

The derivative of the log-likelihood for a deep generative model can be
written as:
\begin{eqnarray}
\frac{\partial\log P_{\mathbf{\theta}}(\mathbf{x})}{\partial\theta}
&=&
\frac{\sum_\h 
\frac{\partial
P_{\theta_{I}}(\mathbf{x}|\mathbf{h})}{\partial\theta}P_{\theta_J}(\h)
+
\sum_\h P_{\theta_I}(\x|\h)\frac{\partial
P_{\theta_J}(\h)}{\partial\theta}
}{P_\theta(\x)}
\\&= & \sum_{\h}\frac{\partial\log P_{\theta_{I}}(\mathbf{x}|\mathbf{h})}{\partial\theta}P_{\theta}(\mathbf{h}|\mathbf{x})
+
\sum_{\h}\frac{\partial\log P_{\theta_{J}}(\mathbf{h})}{\partial\theta}P_{\theta}(\mathbf{h}|\mathbf{x})
\label{eq:globalgradient}
\end{eqnarray}
by rewriting $P_\theta(\h)/P_\theta(\x)=P_\theta(\h|\x)/P_\theta(\x|\h)$.
The derivative w.r.t.\ a given
component $\theta_i$ of $\theta$
simplifies because $\theta_{i}$ is either a parameter of
$P_{\theta_{I}}(\mathbf{x}|\mathbf{h})$ when $i\in I$, or a parameter
of $P_{\theta_{J}}(\mathbf{h})$ when $i\in J$:
\begin{align}
\forall i  \in  I,\qquad
&
\frac{\partial\log
P_{\mathbf{\theta}}(\mathbf{x})}{\partial\theta_{i}}=\sum_{\h}\frac{\partial\log
P_{\theta_{I}}(\mathbf{x}|\mathbf{h})}{\partial\theta_{i}}P_{\theta_I,\theta_J}(\mathbf{h}|\mathbf{x}),\label{eq:dgm-conditional-gradient}
\\
\forall i \in J,\qquad
&
\frac{\partial\log
P_{\mathbf{\theta}}(\mathbf{x})}{\partial\theta_{i}}=\sum_{\h}\frac{\partial\log
P_{\theta_{J}}(\mathbf{h})}{\partial\theta_{i}}P_{\theta_I,\theta_J}(\mathbf{h}|\mathbf{x}).\label{eq:dgm-marginal-gradient}
\end{align}
Unfortunately, this gradient ascent procedure is generally intractable,
because it requires sampling from $P_{\theta_I,\theta_J}(\mathbf{h}|\mathbf{x})$
(where both the upper layer and lower layer influence $\h$)
to perform inference in the deep model.

\section{Layer-wise deep learning\label{sec:Layer-wise-deep-learning}}

\subsection{A theoretical guarantee}
\label{sec:mainthm}

\NDY{Didn't find much use for BOUB. Still hesitating between BOBL and
BOLL.}

We now present a training procedure that works successively on each
layer.  First we train $\theta_I$ together with a conditional model
$q(\h|\x)$ for the latent variable knowing the data. This step involves
only the bottom part of the model and is thus often tractable.  This allows to infer a new target
distribution for $\h$, on which the upper layers can then be trained.

This procedure singles out a particular setting $\hat\theta_I$ for the
bottom layer of a deep architecture, based on an optimistic assumption of
what the upper layers may be able to do (cf.\
Proposition~\ref{prop:Q=qD}).

Under this procedure, Theorem~\ref{thm:main} states that it is possible
to obtain a validation that the parameter $\hat\theta_I$ for the bottom layer
was optimal, provided the rest of the training goes well.
Namely,
\emph{if} the target distribution for $\h$ can be realized or well
approximated by some value of the parameters $\theta_J$ of the top
layers, and if $\theta_I$ was obtained using a rich enough conditional
model $q(\h|\x)$, then $(\theta_I,\theta_J)$ is guaranteed to be globally
optimal.

\begin{thm}
\label{thm:main}
Suppose the parameters $\theta_I$ of the bottom layer are trained by
\begin{equation}
\label{eq:bottomtraining}
(\hat\theta_I,\hat q) \deq \argmax_{\theta_I,q} \E_{\mathbf{x}\sim P_{\mathcal{D}}}
\left[
\log \sum_\h P_{\theta_I}(\x|\h) \, q_\D(\h)
\right]
\end{equation}
where the arg max runs over \emph{all} conditional probability distributions
$q(\h|\x)$ and where
\begin{equation}
\label{eq:qd}
q_\D(\h)\deq \sum_{\tilde\x} q(\h|\tilde\x) P_\D(\tilde\x)
\end{equation}
with $P_\D$ the observed data distribution.

\NDY{I think the new presentation's more punchy.}\NDL{Agreed.}

We call the optimal $\hat\theta_I$ the \emph{best optimistic lower
layer} (BOLL).
Let $\hat q_\D(\h)$ be the distribution on $\h$ associated with the
optimal $\hat q$.
Then:
\begin{itemize}
\item If the top layers can be trained to reproduce $\hat q_\D(\h)$
perfectly,
i.e., if there exists a parameter $\hat\theta_J$ for the top layers
such that the distribution $P_{\hat\theta_J}(\h)$ is equal to $\hat
q_\D(\h)$, then the parameters obtained are globally optimal:
\[
(\hat\theta_I,\hat\theta_J)=(\theta_I^\ast,\theta_J^\ast)
\]

\item Whatever parameter value $\theta_J$ is used on the top layers in
conjunction with the BOLL $\hat \theta_I$, the
difference in performance \eqref{eq:objective} between $(\hat
\theta_I,\theta_J)$ and the global optimum
$(\theta_I^\ast,\theta_J^\ast)$ is at most
the Kullback--Leibler divergence $\KL{\hat q_\D(\h)}{P_{\theta_J}(\h)}$
between $\hat q_\D(\h)$ and $P_{\theta_J}(\h)$.
\end{itemize}

\end{thm}

This theorem strongly suggests using $\hat q_\D(\h)$ as the target
distribution for the top layers, i.e., looking for the value
$\hat\theta_J$ best approximating $\hat q_\D(\h)$:
\begin{equation}
\label{eq:toptraining}
\hat\theta_J \deq \argmin_{\theta_J} \KL{\hat
q_\D(\h)}{P_{\theta_J}(\h)}=\argmax_{\theta_J} \E_{\h\sim \hat q_\D}
\log P_{\theta_J}(\h)
\end{equation}
which thus takes the same form as the original problem.
Then the same scheme may be used recursively to train
the top layers. A final fine-tuning phase may be helpful, see
Section~\ref{sec:finetune}.

Note that when the top layers fail to approximate $\hat
q_\D$ perfectly, the loss of performance depends only on the observed difference
between $\hat
q_\D$ and $P_{\hat\theta_J}$, and not on
the unknown global optimum $(\theta_I^\ast,\theta_J^\ast)$. Beware that,
unfortunately,
this bound relies on \emph{perfect} layer-wise training of the bottom layer, i.e., on
$\hat q$ being the optimum of the criterion~\eqref{eq:bottomtraining}
optimized over all possible conditional distributions $q$; otherwise it is a priori not valid.

In practice the supremum on $q$ will always be
taken over a restricted set of conditional distributions $q(\h|\x)$,
rather than the set of all possible distributions on $\h$ for each $\x$.
Thus, this theorem is an
idealized version of practice (though
Remark~\ref{rem:qfamily} below mitigates this). This still suggests a clear
strategy to separate the deep optimization problem into two subproblems
to be solved sequentially:

\begin{enumerate}
\item Train the parameters $\theta_{I}$ of the bottom layer
after \eqref{eq:bottomtraining}, using a model $q(\h|\x)$ as wide as
possible, to approximate the BOLL $\hat\theta_{I}$.
\item Infer the corresponding distribution of $\h$ by \eqref{eq:qd} and
train the upper part of the model as best as possible to approximate this
distribution.

\end{enumerate}

Then, provided learning is successful in both instances,
the result is close to optimal.

Auto-encoders can be shown to implement an approximation of this
procedure, in which only the terms $\x=\tilde \x$ are kept in
\eqref{eq:bottomtraining}--\eqref{eq:qd}
(Section~\ref{sec:AA}).

This scheme is designed with in mind a situation in which the upper
layers get progessively simpler. Indeed, if the layer for $\h$ is as wide
as the layer for $\x$ and if $P(\x|\h)$ can learn the identity, then the
procedure in Theorem~\ref{thm:main} just transfers the problem
unchanged one layer up.

This theorem strongly suggests decoupling the inference and
generative models $q(\h|\x)$ and $P(\x|\h)$, and using a rich conditional model $q(\h|\x)$,
contrary, e.g., to common practice in auto-encoders\footnote{
Attempts to prevent auto-encoders from learning the
identity (which is completely justifiable) often result in an even more
constrained inference model, e.g.,
tied weights, or sparsity constraints on the hidden representation.}.  Indeed
the experiments of Section~\ref{sec:exp} confirm
that using a more expressive $q(\h|\x)$ yields improved values of
$\theta$.

Importantly, $q(\h|\x)$ is only used as an auxiliary prop for
solving the optimization problem~\eqref{eq:objective} over $\theta$ and
is not part of the final generative model, so that using a richer
$q(\h|\x)$ to reach a better value of $\theta$ is not simply changing to
a larger model.
Thus,
using a richer inference model $q(\h|\x)$ should not pose too much risk
of overfitting because the regularization properties of the model
come mainly from the choice of the generative model family $(\theta)$.

The criterion proposed in \eqref{eq:bottomtraining} is of particular
relevance to representation learning where the goal is not to learn a
generative model, but to learn a useful representation of the data. In
this setting, training an upper layer model $P(\h)$ becomes irrelevant
because we are not interested in the generative model itself.  What
matters in representation learning is that the lower layer (i.e.,
$P(\x|\h)$ and $q(\h|\x)$) is optimal for \emph{some}
model of $P(\h)$, left unspecified.

We now proceed, by steps, to the proof of Theorem~\ref{thm:main}. This
will be the occasion to introduce some concepts used later in the experimental
setting.

\subsection{The Best Latent Marginal Upper Bound}

One way to evaluate a parameter $\theta_I$ for the bottom layer without
training the whole architecture is to be optimistic: assume that the top
layers will be able to produce the probability distribution for $\h$ that
gives the best results if used together with $P_{\theta_I}(\x|\h)$. This
leads to the following.

\begin{defi}
Let $\theta_I$ be a value of the bottom layer parameters. The
\emph{best latent marginal} (BLM) for $\theta_I$ is the probability
distribution $Q$ on $\h$ maximizing the log-likelihood:
\begin{equation}
\label{eq:BLM}
\BLM{\theta_I}{\D}\deq \argmax_Q \E_{\x\sim P_\D}
\left[
\log \sum_\h P_{\theta_I}(\x|\h) Q(\h)
\right]
\end{equation}
where the arg max runs over the set of all probability distributions over
$\h$. The \emph{BLM upper bound} is the corresponding
log-likelihood value:
\begin{equation}
\label{eq:defU}
\U_\D(\theta_I)\deq \max_Q \E_{\x\sim P_\D}
\left[
\log \sum_\h P_{\theta_I}(\x|\h) Q(\h)
\right]
\end{equation}

\end{defi}

The BLM upper bound $\mathcal{U}_{\mathcal{D}}(\theta_{I})$ is the least
upper bound on the log-likelihood of the deep generative model on the
dataset $\mathcal{D}$ if $\theta_{I}$ is used for the bottom layer.  
$\mathcal{U}_\D(\theta_I)$ is only an upper bound of the actual
performance of $\theta_I$, because subsequent
training of $P_{\theta_J}(\mathbf{h})$ may be suboptimal: the best latent
marginal $\BLM{\theta_I}{\D}(\h)$ may not be representable as
$P_{\theta_J}(\mathbf{h})$ for $\theta_J$ in the model,
or the training of
$P_{\theta_J}(\mathbf{h})$ itself may not converge to the best solution.

Note that the arg max in \eqref{eq:BLM} is concave in $Q$, so that in
typical situations 
the BLM is unique---except in degenerate cases such as when two values of
$\h$ define the same $P_{\theta_I}(\x|\h)$).

\del{
We now show that the criterion~\eqref{eq:bottomtraining} proposed in
Theorem~\ref{thm:main} to train the lower part of a deep generative model
coincides with the BLM upper bound.}\NDY{Really redundant with the first
sentence in the 
statement below!}

\begin{prop}
\label{prop:Q=qD}
The
criterion~\eqref{eq:bottomtraining} used in Theorem~\ref{thm:main} for
training the bottom layer coincides with the BLM upper bound:
\begin{equation}
\label{eq:BLMandqD}
\U_\D(\theta_I)=\max_q \E_{\x\sim P_\D}
\left[
\log \sum_\h P_{\theta_I}(\x|\h) q_\D(\h)
\right]
\end{equation}
where the maximum runs over all conditional probability distributions
$q(\h|\x)$.
In particular the BOLL $\hat\theta_I$ selected in Theorem~\ref{thm:main}
is
\begin{equation}
\hat\theta_I=\argmax_{\theta_I}\, \U_\D(\theta_I)
\end{equation}
and the target distribution $\hat q_\D(\h)$ in
Theorem~\ref{thm:main} is the best latent marginal $\BLM{\hat\theta_I}{\D}$.
\end{prop}

Thus the BOLL $\hat\theta_I$ is the best bottom layer setting if one
uses an optimistic criterion for assessing the bottom layer, hence the
name ``best optimistic lower layer''.

\begin{proof}
\NDL{\emph{proofchecking:} Ok.}
Any distribution $Q$ over $\h$ can be written as $q_\D$ for some conditional
distribution $q(\h|\x)$, for instance by defining $q(\h|\x)=Q(\h)$ for every
$\x$ in the dataset. In particular this is the case for the best
latent marginal $\BLM{\theta_I}{\D}$.

Consequently the maxima in~\eqref{eq:BLMandqD} and in~\eqref{eq:defU} are
taken on the same set and coincide.
\end{proof}

The argument that any distribution is of the form $q_\D$ may look
disappointing: why choose this particular form? In
Section~\ref{sec:whyqd} we show how
writing distributions over $\h$ as $q_\D$ for some conditional
distribution $q(\h|\x)$ may help to maximize data log-likelihood, by
quantifiably incorporating information from the data
(Proposition~\ref{prop:datainc}). Moreover, the 
bound on loss of performance (second part of Theorem~\ref{thm:main}) when
the upper layers do not match the BLM crucially relies on the properties of
$\hat q_\D$.
A more
practical argument for using $q_\D$
is that optimizing both $\theta_I$ and the full distribution of the hidden
variable $\h$ at the
same time is just as difficult as optimizing the whole network, whereas
the deep architectures currently in use already train a model of $\x$ knowing
$\h$ and of $\h$ knowing $\x$ at the same time.

\begin{rem}
\label{rem:qfamily}
For Theorem~\ref{thm:main} to hold, it is not necessary to optimize
over all possible conditional probability distributions $q(\h|\x)$
(which is a set of very large dimension). As can be seen from the proof
above it is enough to optimize over a
family $q(\h|\x)\in \mathcal{Q}$ such that every (non-conditional) distribution on $\h$ can be
represented (or well approximated) as $q_\D(\h)$ for some $q\in
\mathcal{Q}$.
\end{rem}

Let us now go on with the proof of Theorem~\ref{thm:main}.

\begin{prop}
\label{prop:BLMtraining}
Set the bottom layer parameters to the BOLL
\begin{align}
\hat\theta_I
&=
\argmax_{\theta_I} \,\U_\D(\theta_I)
\end{align}
and let $\hat Q$ be the corresponding best latent marginal.

Assume that subsequent training of the top layers using
$\hat Q$ as the target distribution for $\h$, is successful, i.e.,
there exists a $\theta_J$ such that $\hat Q(\h)=P_{\theta_J}(\h)$.

Then $\hat\theta_I=\theta_I^\ast$.
\end{prop}

\begin{proof}
\NDL{\emph{proofchecking:} Ok.}
Define the \emph{in-model} BLM upper bound as
\begin{equation}
\UIM_\D(\theta_I)\deq \max_{\theta_J} \E_{\x\sim P_\D}
\left[
\log \sum_\h P_{\theta_I}(\x|\h) P_{\theta_J}(\h)
\right]
\label{eq:inmodelbound}
\end{equation}

By definition, the global optimum $\theta_I^\ast$ for the parameters of
the whole architecture is given by
$\theta_I^\ast=\argmax_{\theta_I} \UIM_\D(\theta_I)$. 

Obviously, for any value $\theta_I$ we have $\UIM_\D(\theta_I)\leq
\U_\D(\theta_I)$ since the argmax is taken over a more restricted set.
Then, in turn, $\U_\D(\theta_I)\leq
\U_\D(\hat\theta_I)$ by definition of $\hat \theta_I$.

By our assumption, the BLM $\hat Q$ for $\hat\theta_I$ happens to lie in
the model: $\hat Q(\h)=P_{\theta_J}(\h)$. This implies that 
$\U_\D(\hat\theta_I)=\UIM_\D(\hat\theta_I)$.

Combining, we get that $\UIM_\D(\theta_I)\leq \UIM_\D(\hat\theta_I)$ for any
$\theta_I$.
Thus $\hat\theta_I$ maximizes $\UIM_\D(\theta_I)$, and is thus equal to
$\theta_I^\ast$.
\end{proof}

The first part of Theorem~\ref{thm:main} then results from the combination of
Propositions~\ref{prop:BLMtraining} and~\ref{prop:Q=qD}\del{, for the case
when the upper layers can be trained to reproduce $q_\D$ perfectly}.
\NDL{\emph{proofchecking:} Ok for proof of theorem 1. By prop 3 Theorem 1 proposes to maximize the BLM upper bound. By prop 5 and with the assumption of perfect reproduction of $Q(\h)$ Theorem 1 is proven to give optimal in-model parameters.}

We now give a bound on the loss of performance in case further training
of the upper layers fails to reproduce the BLM. This will complete the
proof of Theorem~\ref{thm:main}. We will make use of a special
optimality property of distributions of the form $q_\D(\h)$, namely,
Proposition~\ref{prop:datainc}, whose proof is postponed to
Section~\ref{sec:whyqd}.

\begin{prop}
\label{prop:quant}
Keep the notation of Theorem~\ref{thm:main}. In the case when
$P_{\theta_J}(\h)$ fails to reproduce $\hat q_\D(\h)$ exactly, the
loss of performance of $(\hat \theta_I,\theta_J)$ with respect to
the global optimum $(\theta^\ast_I,\theta^\ast_J)$ is at most
\begin{equation}
\KL{P_\D(\x)}{P_{\hat\theta_I,\theta_J}(\x)}
-
\KL{P_\D(\x)}{\hat q_{\D,\hat \theta_I}(\x)}
\end{equation}
where $\hat q_{\D,\hat \theta_I}(\x)\deq\sum_\h
P_{\hat\theta_I}(\x|\h)\hat q_\D(\h)$ is the distribution on $\x$
obtained by using the BLM.

This quantity is in turn at most
\begin{equation}
\label{eq:explicitbound}
\KL{\hat
q_\D(\h)}{P_{\theta_J}(\h)}
\end{equation}
which is thus also a bound on the loss of 
performance of $(\hat \theta_I,\theta_J)$ with respect to
$(\theta^\ast_I,\theta^\ast_J)$.
\end{prop}

Note that these estimates do not depend on the unkown global optimum $\theta^*$.

Importantly, this bound is \emph{not} valid if $\hat q$ has not been perfectly
optimized over all possible conditional distributions $q(\h|\x)$. Thus it
should not be used blindly to get a performance bound, since heuristics
will always be used to find $\hat q$. Therefore, it may have only limited
practical relevance. In practice the real loss may both be larger than
this bound because $q$ has been optimized over a smaller set, and smaller
because we are comparing to the BLM upper bound which is an
optimistic assessment.

\begin{proof}
From~\eqref{eq:objective} and \eqref{eq:KLobjective}, the difference in log-likelihood performance between any two
distributions $p_1(\x)$ and $p_2(\x)$ is equal to
$\KL{P_\D}{p_1}-\KL{P_\D}{p_2}$.

For simplicity, denote 
\begin{align*}
p_1(\x)&=P_{\hat\theta_I,\theta_J}(\x)=\sum_\h
P_{\hat\theta_I}(\x|\h)P_{\theta_J}(\h)
\\p_2(\x)&=P_{\theta_I^\ast,\theta_J^\ast}(\x)
=\sum_\h P_{\theta_I^\ast}(\x|\h) P_{\theta_J^\ast}(\h)
\\p_3(\x)&=\sum_\h
P_{\hat\theta_I}(\x|\h)\hat q_\D(\h)
\end{align*}

We want to compare $p_1$ and
$p_2$.

Define the in-model upper bound 
$\UIM_\D(\theta_I)$ as in \eqref{eq:inmodelbound} above. Then we have
$\theta_I^\ast=\argmax_{\theta_I} \UIM_\D(\theta_I)$ and
$\hat\theta_I=\argmax_{\theta_I} \U_\D(\theta_I)$. Since $\UIM_\D\leq
\U_\D$, we have $\UIM_\D(\theta_I^\ast)\leq \U_D(\hat\theta_I)$. The
BLM upper bound $\U_D(\hat\theta_I)$ is attained when we use $\hat q_\D$
as the distribution for $\h$, so $\UIM_\D(\theta_I^\ast)\leq
\U_D(\hat\theta_I)$ means that the performance of $p_3$ is better than
the performance of $p_2$:
\[
\KL{P_\D}{p_3}\leq \KL{P_\D}{p_2}
\]
(inequalities hold in the reverse order for data log-likelihood).

Now by definition of the optimum $\theta^\ast$, the distribution $p_2$ is
better than $p_1$: $\KL{P_\D}{p_2}\leq \KL{P_\D}{p_1}$. Consequently, the
difference in performance between $p_2$ and $p_1$ (whether expressed in
data log-likelihood or in Kullback--Leibler divergence) is smaller than the
difference in performance between $p_3$ and $p_1$, which is the
difference of 
Kullback--Leibler divergences
appearing in the proposition.

\NDL{\emph{proofchecking:} Ok to this point. Very nice argument.}

Let us now evaluate more precisely the loss of $p_1$ with respect to
$p_3$. By abuse of notation we will indifferently denote $p_1(\h)$ and
$p_1(\x)$, it being understood that one is obtained from the other through
$P_{\hat\theta_I}(\x|\h)$, and likewise for $p_3$ (with the same
$\hat\theta_I$).

For any distributions $p_1$ and $p_3$ 
the
loss of performance of $p_1$ w.r.t.\ $p_3$ satisfies
\[
\E_{\x\sim P_\D} \log p_3(\x)-\E_{\x\sim P_\D} \log p_1(\x)
=
\E_{\x\sim P_\D} \left[
\log \frac{\sum_\h P_{\hat\theta_I}(\x|\h)p_3(\h)}
{\sum_\h P_{\hat\theta_I}(\x|\h)p_1(\h)}
\right]
\]
and by the log sum inequality $\log (\sum a_i/\sum b_i)\leq \frac{1}{\sum
a_i} \sum a_i \log (a_i/b_i)$ \cite[Theorem 2.7.1]{CoverThomas} we get
\begin{align*}
&\E_{\x\sim P_\D} \log p_3(\x)-\E_{\x\sim P_\D} \log p_1(\x)
\\&\leq
\E_{\x\sim P_\D} \left[\frac{1}{\sum_\h P_{\hat\theta_I}(\x|\h)p_3(\h)}
\sum_\h  P_{\hat\theta_I}(\x|\h)p_3(\h)\log \frac{
P_{\hat\theta_I}(\x|\h)p_3(\h)}{P_{\hat\theta_I}(\x|\h)p_1(\h)}
\right]
\\&= \E_{\x\sim P_\D}\left[
\frac{1}{p_3(\x)}\sum_\h p_3(\x,\h) \log \frac{p_3(\h)}{p_1(\h)}
\right]
\\&= \E_{\x\sim P_\D}\left[
\sum_\h p_3(\h|\x) \log \frac{p_3(\h)}{p_1(\h)}
\right]
\\&=\E_{\x\sim
P_\D} \E_{\h\sim p_3(\h|\x)} \left[\log \frac{p_3(\h)}{p_1(\h)}\right]
\end{align*}

Given a probability $p_3$ on $(\x,\h)$, 
the law on $\h$ obtained by taking an $\x$ according
to $P_\D$, then taking an $\h$ according to $p_3(\h|\x)$, is generally
not equal to $p_3(\h)$. However, here $p_3$ is equal to the BLM $\hat q_D$,
and by Proposition~\ref{prop:datainc} below the BLM has exactly this property
(which characterizes the log-likelihood extrema).
Thus thanks to
Proposition~\ref{prop:datainc} we have
\[
\E_{\x\sim
P_\D} \E_{\h\sim \hat q_\D(\h|\x)} \left[\log\frac{\hat q_\D(\h)}{p_1(\h)}\right]
=\E_{\h\sim \hat q_\D} \left[\log\frac{\hat q_\D(\h)}{p_1(\h)}\right]
=\KL{\hat q_\D(\h)}{p_1(\h)}
\]
which concludes the argument.
\NDL{\emph{proofchecking: } rechecked and ok. re-re-checked that this is ok for any $\theta_I$.}
\end{proof}

\subsection{Relation with Stacked RBMs}
\label{sec:SRBMs}

Stacked RBMs (SRBMs) \cite{Hinton2006,Bengio2007a,Larochelle2009} are deep generative
models trained by stacking restricted Boltzmann machines (RBMs)
\cite{Smolensky1986}.

A RBM uses a single set of parameters to represent a distribution on
pairs $(\x,\h)$. Similarly to our approach, stacked RBMs are trained in a
greedy layer-wise fashion: one starts by training the distribution of the
bottom RBM to approximate the distribution of $\x$. To do so, distributions 
$P_{\theta_I}(\mathbf{x}|\mathbf{h})$ and $Q_{\theta_I}(\mathbf{h}|\mathbf{x})$ are learned
jointly using a \emph{single} set of parameters $\theta_I$. Then a target
distribution for $\h$ is defined as $\sum_\x Q_{\theta_I}(\h|\x)
P_\D(\x)$ (similarly to \eqref{eq:qd}) and the top layers are trained
recursively on this distribution.

In the final generative model, the full top RBM is used on the top layer to
provide a distribution for $\h$, then the bottom RBMs are used only for
the generation of $\x$ knowing $\h$. (Therefore the
$\h$-biases of the bottom RBMs are never used in the final generative
model.)

Thus, in contrast with our approach, $P_{\theta_I}(\mathbf{x}|\mathbf{h})$
and $Q_{\theta_I}(\mathbf{h}|\mathbf{x})$ are not trained to maximize the
least upper
bound of the likelihood of the full deep generative model
but are trained to maximize the likelihood of a single RBM.

This procedure has been shown to be equivalent to maximizing the likelihood
of a deep generative model with infinitely many layers where the weights
are all tied \cite{Hinton2006}. 
The latter can be interpreted as an assumption
on the future value of $P(\mathbf{h})$, which is unknown when learning
the first layer. As such, SRBMs make a different assumption about
the future $P(\mathbf{h})$ than the one made in~\eqref{eq:bottomtraining}.

With respect to this, the comparison of gradient ascents is instructive:
the gradient ascent for training the bottom RBM takes a form reminiscent
of gradient ascent of the global generative
model~\eqref{eq:globalgradient} but in which the dependency of $P(\h)$ on the
upper layers $\theta_J$ is ignored, and instead the distribution
$P(\h)$ is tied to $\theta_I$ because the RBM uses a single parameter set
for both.

When adding a new layer on top of a trained RBM, if the initialization is
set to an upside down version of the current RBM (which can be seen as
``unrolling'' one step of Gibbs sampling), the new deep model still matches
the special infinite deep generative model with tied weights mentioned
above. Starting training of the upper layer from this initialization
guarantees that the new layer can only increase the
likelihood \cite{Hinton2006}. However, this result is only known to hold
for two layers; with more layers, it is only known that adding
layers increases a \emph{bound} on the likelihood \cite{Hinton2006}.

In our approach, the perspective is different. During the training
of lower layers, we consider the best possible model for the hidden
variable.
Because of errors which are bound to occur in approximation
and optimization during the training of the model for $P(\h)$,
the likelihood associated with an optimal upper model (the BLM upper bound)
is expected to \emph{decrease} each time we actually take another lower layer into account:
At each new layer, errors in approximation or optimization
occur so that the final likelihood of the training
set will be smaller
than the upper bound. (On the other way these limitations might actually
improve performance on a test set, see the discussion about
regularization in Section~\ref{sec:exp}.)

In \cite{LeRoux2008} a training criterion is suggested for SRBMs which is
reminiscent of a BLM with tied weights for the inference and generative
parts (and therefore without the BLM optimality guarantee), see also
Section~\ref{sec:fromSRBMstoAEs}.

\subsection{\label{sec:AA}Relation with Auto-Encoders}

Since the introduction of deep neural networks, auto-encoders \cite{Vincent2008}
have been considered a credible alternative to stacked RBMs and have
been shown to have almost identical performance on several tasks
\cite{Larochelle2007}.

Auto-encoders are trained by stacking auto-associators \cite{Bourlard1988}
trained with backpropagation. Namely: we start with a three-layer network
$\x\mapsto \h^{(1)} \mapsto \x$ trained by backpropagation to reproduce the
data; this provides two conditional distributions $P(\h^{(1)}|\x)$ and
$P(\x|\h^{(1)})$. Then in turn, another auto-associator is trained as a
three-layer network $\h^{(1)} \mapsto \h^{(2)} \mapsto\h^{(1)}$, to
reproduce the distribution $P(\h^{(1)}|\x)$ on $\h^{(1)}$, etc.

So as in the learning of SRBMs, auto-encoder training is performed in a
greedy layer-wise manner, but with a different criterion: the
reconstruction error.

Note that after the auto-encoder has been trained, the deep generative
model is incomplete because it lacks a generative model for the distribution
$P(\h^{k_\mathrm{max}})$ of the deepest hidden variable, which the
auto-encoder does not provide\footnote{Auto-associators can in fact be
used as valid generative models from which sampling is possible
\cite{Rifai2012} in the setting of manifold learning but this is beyond
the scope of this article.}.
One possibility is to learn the top layer with an RBM, which then completes
the generative model.

Concerning the theoretical soundness of stacking auto-associators for
training deep generative models, it is known that the training of
auto-associators is an approximation of the training of RBMs in which
only the largest term of an expansion of the log-likelihood is kept
\cite{Bengio2009}. In this sense, SRBM and stacked
auto-associator training approximate each other (see also
Section~\ref{sec:fromSRBMstoAEs}).

Our approach gives a new understanding of auto-encoders as the lower part
of a deep generative model, because they are trained to maximize a \emph{lower
bound} of \eqref{eq:bottomtraining}, as follows.

To fix ideas, let us consider for~\eqref{eq:bottomtraining} a particular class of conditional
distributions $q(\mathbf{h}|\mathbf{x})$ commonly used in
auto-associators.
Namely, let us
parametrize $q$ as $q_{\xi}$ with
\begin{eqnarray}
q_{\xi}(\mathbf{h}|\mathbf{x}) & = & \prod_{j}q_{\xi}(h_{j}|\mathbf{x})\\
q_{\xi}(h_{j}|\mathbf{x}) & = & \sigm({\textstyle\sum}_{i}x_{i}w_{ij}+b_{j})
\end{eqnarray}
where the parameter vector is $\xi=\{\mathbf{W},\mathbf{b}\}$ and
$\sigm(\cdot)$ is the sigmoid function.

Given a conditional
distribution $q(\h|\x)$ as in Theorem~\ref{thm:main}, let us expand the distribution on $\x$ obtained
from $P_{\theta_I}(\x|\h)$ and $q_\D(\h)$:
\begin{align}
P(\x)&= \sum_\h P_{\theta_I}(\x|\h) q_\D(\h)
\\
&= \sum_\h P_{\theta_I}(\x|\h) \sum_{\tilde \x} q(\h|\tilde \x) P_\D(\tilde \x)
\label{eq:wide-ae}
\end{align}
where as usual $P_\D$ is the data distribution. Keeping only the terms
$\x=\tilde \x$ in this expression we see that
\begin{equation}
P(\x)\geq \sum_\h P_{\theta_I}(\x|\h) q(\h|\x) P_\D(\x)
\label{eq:auto-associator-proba}
\end{equation}
Taking the sum of likelihoods over $\x$ in the dataset,
this corresponds to
the criterion maximized by auto-associators when they are
considered from a probabilistic perspective%
\footnote{In all fairness, the training of auto-associators by backpropagation, in probabilistic
terms, consists in the maximization of
$P(\mathbf{y}|\mathbf{x})P_\D(\x) = o(\x)P_\D(\x)$ with
$\mathbf{y}=\mathbf{x}$ \cite{Buntine1991},
where
$o$ is the output function of the neural network. In this perspective,
the hidden variable $\h$ is not considered as a random variable but as an
intermediate value in the form of $P(\mathbf{y}|\mathbf{x})$. Here, we
introduce $\h$ as an intermediate random variable as in \cite{Neal1990}.
The criterion we wish to maximize is then
$P(\mathbf{y}|\mathbf{x})P_\D(\x)=\sum_{\mathbf{h}}f(\mathbf{y}|\mathbf{h})g(\mathbf{h}|\mathbf{x})P_\D(\x)$,
with $\mathbf{y}=\x$.
Training with
backpropagation can be done by sampling $\mathbf{h}$
from $g(\mathbf{h}|\mathbf{x})$ instead of using the raw activation
value of $g(\mathbf{h}|\mathbf{x})$, but in practice we do not sample
$\h$ as it does not significantly  affect performance.}.
Since moreover optimizing over $q$ as in \eqref{eq:bottomtraining} is more general than optimizing over
the particular class $q_\xi$, we conclude that
\emph{the criterion optimized in auto-associators is a lower bound on the
criterion~\eqref{eq:bottomtraining} proposed in Theorem~\ref{thm:main}}.

\NDY{Should we make this a formal statement?}

Keeping only $\x=\tilde \x$ is justified 
if we assume that inference is an approximation of the inverse
of the generative process%
\footnote{which is a reasonable assumption if we are to perform inference
in any meaningful sense of the word.%
}, that is, $P_{\theta_{I}}(\x|\h)q(\h|\tilde \x)\approx0$ as soon as $\x\neq
\tilde \x$.
Thus under this assumption, both criteria will be close, so
that Theorem~\ref{thm:main} provides a justification for auto-encoder
training in this case. On the other hand, this assumption can be strong:
it implies that no $\h$ can be shared between
different $\x$, so that for instance two observations cannot come from
the same underlying latent variable through a random choice. Depending on
the situation this might be unrealistic. Still, using this as a training
criterion might perform well even if the assumption is not fully
satisfied.

Note that we chose the form of $q_{\xi}(\mathbf{h}|\mathbf{x})$
 to match that of the usual
auto-associator, but of course we could have made a different choice
such as using a multilayer network for $q_{\xi}(\mathbf{h}|\mathbf{x})$
or $P_{\theta_{I}}(\mathbf{x}|\mathbf{h})$. These possibilities
will be explored later in this article.

\subsection{From stacked RBMs to auto-encoders: layer-wise consistency}
\label{sec:fromSRBMstoAEs}

We now show how imposing a ``layer-wise consistency'' constraint on
stacked RBM training leads to the training criterion used in
auto-encoders with tied weights. Some of the material here already
appears in \cite{LeRoux2008}.

Let us call \emph{layer-wise consistent} a layer-wise training procedure in
which each layer determines a value $\theta_I$ for its parameters and
a target distribution $P(\h)$ for the upper layers which are mutually
optimal in the following sense:
if $P(\h)$ is used a
the distribution of the hidden variable, then $\theta_I$ is the bottom
parameter value maximizing data log-likelihood.

The BLM training procedure is, by construction, layer-wise consistent.

Let us try to train stacked RBMs in a layer-wise consistent way. Given a
parameter $\theta_I$, SRBMs use the hidden variable distribution
\begin{equation}
Q_{\D,\theta_I}(\h)=\E_{\x\sim P_\D}\, P_{\theta_I}(\h|\x)
\end{equation}
as the target for the next layer,
where $P_{\theta_I}(\h|\x)$ is the RBM distribution of $\h$ knowing $\x$. The
value $\theta_I$ and this distribution over $\h$ are mutually optimal for each
other if the distribution on $\x$ stemming from this distribution on
$\h$, given by
\begin{align}
P^{(1)}_{\theta_I}(\x) &= \E_{\h\sim
Q_{\D,\theta_I}(\h)}\,P_{\theta_I}(\x|\h)
\\&= \sum_\h P_{\theta_I}(\x|\h) \sum_{\tilde \x}
P_{\theta_I}(\h|\tilde \x) P_\D(\tilde \x)
\end{align}
maximizes log-likelihood, i.e., 
\begin{equation}
\label{eq:layerconsis}
\theta_I=\argmin \KL{P_\D(\x)}{P^{(1)}_{\theta_I}(\x)}
\end{equation}
The distribution $P^{(1)}_{\theta_I}(\x)$ is the one obtained from
the data after one ``forward-backward'' step of Gibbs sampling
$\x\to\h\to \x$ (cf.~\cite{LeRoux2008}).

But $P^{(1)}_{\theta_I}(\x)$ is also equal to the
distribution \eqref{eq:wide-ae} for an auto-encoder with tied weights. So
the layer-wise consistency criterion for RBMs coincides with 
tied-weights auto-encoder training, up to the approximation
that in practice auto-encoders retain only the terms $\x=\tilde \x$ in
the above (Section~\ref{sec:AA}).

\newcommand{\PRBM}{P^{\mathrm{RBM}}}
\newcommand{\thetaRBM}{\theta^{\mathrm{RBM}}}

On the other hand, stacked RBM training trains the
parameter $\theta_I$ to approximate the data
distribution by the RBM distribution:
\begin{equation}
\label{eq:SRBMtraining}
\thetaRBM_I= \argmin_{\theta_I} \KL{P_\D(\x)}{\PRBM_{\theta_I}(\x)}
\end{equation}
where $\PRBM_{\theta_I}$ is the probability distribution of the RBM
with parameter $\theta_I$, i.e.\ the probability distribution after an
infinite number of Gibbs samplings from the data.

Thus, stacked RBM training and tied-weight auto-encoder training can be
seen as two approximations to the layer-wise consistent optimization
problem~\eqref{eq:layerconsis}, one using the full RBM distribution
$\PRBM_{\theta_I}$ instead of $P^{(1)}_{\theta_I}$ and
the other using $\x=\tilde\x$ in $P^{(1)}_{\theta_I}$.

It is not clear to us to which extent the criteria \eqref{eq:layerconsis}
using $P^{(1)}_{\theta_I}$ and $\eqref{eq:SRBMtraining}$ using
$\PRBM_{\theta_I}$ actually yield different values for the optimal
$\theta_I$: although these two optimization criteria are different
(unless RBM Gibbs sampling converges in one step), it
might be that the optimal $\theta_I$ is the same (in which case SRBM
training would be layer-wise consistent), though this seems
unlikely.

The $\theta_I$ obtained from the layer-wise consistent criterion
\eqref{eq:layerconsis} using $P^{(1)}_{\theta_I}(\x)$ will always perform
at least as well as standard SRBM training if the upper layers match the
target distribution on $\h$ perfectly---this follows from its very
definition.

Nonetheless, it is not clear whether layer-wise consistency is always a desirable
property. In SRBM training, replacing the RBM distribution over $\h$ with
the one obtained from the data seemingly breaks layer-wise consistency,
but at the same time it always \emph{improves} data log-likelihood (as a
consequence of Proposition~\ref{prop:datainc} below).

For non-layer-wise consistent training procedures,
fine-tuning of $\theta_I$ after more layers have
been trained would improve performance.
Layer-wise consistent procedures may require
this as well in case the
upper layers do not match the target distribution on $\h$ (while
non-layer-wise consistent procedures would require
this even with
perfect upper layer training).

\subsection{Relation to fine-tuning}
\label{sec:finetune}

When the approach presented in 
Section~\ref{sec:Layer-wise-deep-learning} is used recursively to train
deep generative models with several layers using the criterion
\eqref{eq:bottomtraining}, irrecoverable losses may be
incurred at each step: first, because the optimization problem
\eqref{eq:bottomtraining} may be imperfectly solved, and, second, because
each layer was trained using a BLM assumption about what upper layers are
able to do, and subsequent upper layer training may not match the BLM.
Consequently the parameters used for each layer may not be optimal with
respect to each other.
This suggests using a fine-tuning procedure. 

In the case of
auto-encoders, fine-tuning can be done by backpropagation on all
(inference and generative) layers at once (Figure~\ref{fig:fine-tuning}).
This has been shown to improve performance\footnote{The exact likelihood
not being tractable for larger models, it is necessary to rely on a proxy
such as classification performance to evaluate the performance of the
deep network.} in several contexts \cite{Larochelle2009,Hinton2006a},
which confirms the expected gain in performance from recovering earlier
approximation losses.
In principle, there is no limit to the number of layers of an
auto-encoder that could be trained at once by backpropagation, but in
practice training many layers at once results in a difficult optimization
problem with many local minima. Layer-wise training can be seen as a way
of dealing with the issue of local minima, providing a solution close to
a good optimum. This optimum is then reached by global fine-tuning.

\begin{figure}
\begin{centering}
\includegraphics[width=1\columnwidth]{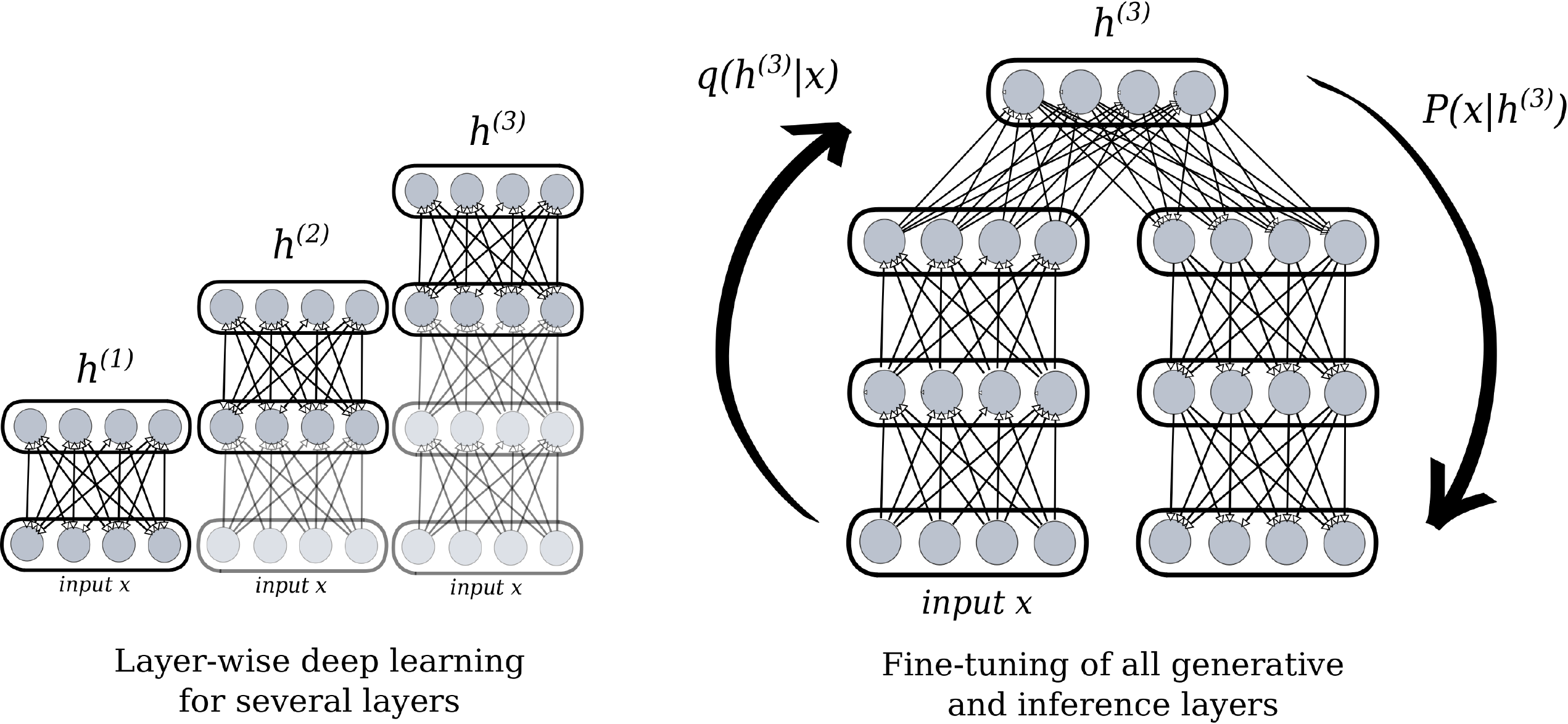}
\par\end{centering}
\caption{\label{fig:fine-tuning}Deep training with fine-tuning.}
\end{figure}

Fine-tuning can be described in the BLM framework as follows: fine-tuning
is the maximization of the BLM upper bound~\eqref{eq:bottomtraining}
where all the layers are considered as one single complex layer
(Figure~\ref{fig:fine-tuning}). In the case of auto-encoders, the
approximation $\x=\tilde\x$ in \eqref{eq:bottomtraining}--\eqref{eq:qd}
is used to help optimization,
 as explained above.

Note that there is no reason to limit fine-tuning to the end of the
layer-wise procedure: fine-tuning may be used at intermediate stages
where any number of layers have been trained.

This fine-tuning procedure was not applied in the experiments below
because our experiments only have one layer for the bottom part of the
model.

As mentioned before, a generative model for the topmost hidden layer
(e.g., an RBM) still needs to be trained to get a complete generative
model after fine-tuning.

\subsection{Data Incorporation: Properties of $q_\D$}
\label{sec:whyqd}

It is not clear why it should be more interesting to work with the
conditional distribution $q(\h|\x)$ and then define a distribution on $\h$
through $q_\D$, rather than working directly with a distribution $Q$ on
$\h$.

The first answer is practical: optimizing on $P_{\theta_I}(\x|\h)$ and on
the distribution of $\h$ simultaneously is just the same as optimizing
over the global network, while on the other hand the currently used deep
architectures provide both $\x|\h$ and $\h|\x$ at the same time.

A second answer is mathematical: $q_\D$ is defined through the dataset
$\D$. Thus by working on $q(\h|\x)$ we can concentrate on the
correspondence between $\h$ and $\x$ and not on the full distribution of
either, and hopefully this correspondence is easier to describe.  Then we
use the dataset $\D$ to provide $q_\D$: so rather than directly
crafting a distribution $Q(\h)$, we use a distribution which
automatically incorporates aspects of the data distribution $\D$ even for
very simple $q$. Hopefully this is better; we now formalize this
argument.

Let us fix the bottom layer parameters $\theta_I$, and consider the
problem of finding the best latent marginal over $\h$, i.e., the $Q$
maximizing the data log-likelihood
\begin{equation}
\label{eq:Qll}
\argmax_Q \E_{\x\sim P_\D}
\left[
\log \sum_\h P_{\theta_I}(\x|\h) Q(\h)
\right]
\end{equation}

Let $Q(\h)$ be a candidate distribution. We might build a better one by
``reflecting the data'' in it. Namely, $Q(\h)$ defines a distribution
$P_{\theta_I}(\x|\h)Q(\h)$ on $(\x,\h)$. This distribution, in turn, defines a
conditional distribution of $\h$ knowing $\x$ in the standard way:
\begin{equation}
\label{eq:defqcond}
Q\cond(\h|\x)\deq \frac{P_{\theta_I}(\x|\h)Q(\h)}{\sum_{\h'}
P_{\theta_I}(\x|\h')Q(\h')}
\end{equation}

We can turn $Q\cond(\h|\x)$ into a new distribution on $\h$ by using the
data distribution:
\begin{equation}
\label{eq:datainc}
Q\cond_\D(\h)\deq \sum_\x Q\cond(\h|\x)P_\D(\x)
\end{equation}
and in general $Q\cond_\D(\h)$ will not coincide with the original
distribution $Q(\h)$, if only because the definition of the former involves
the data whereas $Q$ is arbitrary. We will show that this operation is
always an improvement: $Q\cond_\D(\h)$
always yields a better data log-likelihood than $Q$.

\begin{prop}
\label{prop:datainc}
Let \emph{data incorporation} be the map sending a distribution $Q(\h)$
to $Q\cond_\D(\h)$ defined by~\eqref{eq:defqcond}
and~\eqref{eq:datainc}, where $\theta_I$ is fixed. It has the following properties:
\begin{itemize}
\item Data incorporation always increases the data log-likelihood
\eqref{eq:Qll}.
\item The best latent marginal $\BLM{\theta_I}{\D}$ is a fixed point of
this transformation. 
More precisely, the distributions $Q$ that are fixed points of data incorporation are
exactly the critical points of the data log-likelihood \eqref{eq:Qll}
(by concavity of \eqref{eq:Qll} these critical points are all maxima with the same value). In particular if the BLM is uniquely defined (the
arg max in \eqref{eq:BLM} is unique), then it is the only fixed point of
data incorporation.
\del{More generally, the distributions $Q$ that are critical
points of the log-likelihood~\eqref{eq:Qll} (by concavity all these
critical points are maxima)
are exactly the distributions that are invariant by data
incorporation.}
\item Data incorporation $Q\mapsto Q\cond_\D$ coincides with one step of the
expectation-maximization (EM) algorithm
to maximize data log-likelihood by optimizing over $Q$ for a fixed
$\theta_I$, with $\h$ as the hidden variable.
\end{itemize}
\end{prop}

This can be seen as a justification for constructing the hidden variable
model $Q$ through an inference model $q(\h|\x)$ from the data, which is
the basic approach of auto-encoders and the BLM.
\NDY{The version below was too verbose.}
\del{
Thus, describing the distribution of $\h$ as $q_\D$ for some conditional
distribution $q(\h|\x)$, rather than as an arbitrary distribution, means
we are ``half a step ahead'' in data incorporation. So, even though any
distribution $Q(\h)$ can be represented as $q_\D$ for some $q$ (and
for generic $\theta_I$ the map $Q\mapsto Q\cond_\D(\h)$ is surjective), the
map $Q\mapsto Q\cond_\D$ is ``contractive'' in the sense that it
makes the distribution closer to the optimum, so that if we look for $Q$
of the form $Q\cond_\D$ we are looking in
an interesting part of the search space. In that sense, even starting
with a simple conditional distribution $q(\h|\x)$, the distribution
$q_\D(\h)$ may already have good properties with respect to the data.

This can be seen as further justification for the basic approach of
auto-encoders to build a model for $\h$ knowing $\x$ and feed it the data
distribution.
}

\begin{proof}
Let us first prove the statement about expectation-maximization. Since
the EM algorithm is known to increase data log-likelihood at each step \cite{Dempster1977,Wu1983},
this will prove the first statement as well.

For simplicity let us assume that the data distribution is uniform over
the dataset $\D=(\x_1,\ldots,\x_n)$. (Arbitrary data weights can be approximated
by putting the same observation several times into the data.) The hidden
variable of the EM algorithm will be $\h$, and the parameter over which
the EM optimizes will be the distribution $Q(\h)$ itself. In particular
we keep $\theta_I$ fixed. The distributions $Q$ and $P_{\theta_I}$ define
a distribution $P(\x,\h)\deq P_{\theta_I}(\x|\h)Q(\h)$ over pairs $(\x,\h)$. This
extends to a distribution over $n$-tuples of observations:
\[
P((\x_1,\h_1),\ldots,(\x_n,\h_n))=\prod_i P_{\theta_I}(\x_i|\h_i)Q(\h_i)
\]
and by summing over the states of the hidden variables
\[
P(\x_1,\ldots,\x_n)=\sum_{(\h_1,\ldots,\h_n)}
P((\x_1,\h_1),\ldots,(\x_n,\h_n))
\]

\newcommand{\vx}{\vec{\x}}
\newcommand{\vh}{\vec{\h}}
Denote $\vx=(\x_1,\ldots,\x_n)$ and
$\vh=(\h_1,\ldots,\h_n)$.
One step of the EM algorithm operating with the distribution $Q$ as
parameter, is defined as transforming the current distribution $Q_t$ into
the new distribution
\[
Q_{t+1}=\argmax_{Q} \sum_{\vh} P_t(\vh|\vx)\log P(\vx,\vh)
\]
where $P_t(\vx,\vh)=P_{\theta_I}(\vx|\vh)Q_t(\vh)$ is the distribution
obtained by using $Q_t$ for $\h$, and $P$ the one obtained from the
distribution $Q$ over which we optimize.
\NDL{\emph{proofchecking: }Ok, I recognize EM.}
Let us follow a standard argument for EM
algorithms on $n$-tuples of independent observations:
\begin{align*}
\sum_{\vh} P_t(\vh|\vx)\log P(\vx,\vh)
&=
\sum_{\vh} P_t(\vh|\vx) \log \prod_i P(\x_i,\h_i)
\\&=
\sum_i \sum_{\vh} P_t(\vh|\vx) \log P(\x_i,\h_i)
\end{align*}
Since observations are independent, $P_t(\vh|\vx)$ decomposes as a
product and so
\begin{align*}
\sum_i \sum_{\vh}(\log P(\x_i,\h_i)) P_t(\vh|\vx)
&=
\sum_i \sum_{\h_1,\ldots,\h_n}  (\log P(\x_i,\h_i)) \prod_j P_t(\h_j|\x_j)
\\&=\sum_i \sum_{\h_i}(\log P(\x_i,\h_i)) P_t(\h_i|\x_i)\prod_{j\neq
i}\sum_{\h_j}P_t(\h_j|\x_j)
\end{align*}
\NDL{\emph{proofchecking: }
Is this the argument to transform $\sum \prod$ in $\prod \sum$ ?
\[
=\sum_i \sum_{\h_i}(\log P(\x_i,\h_i)) P_t(\h_i|\x_i)\sum_{\h_1\neq \h_i}P_t(\h_1|\x_j)\sum_{\h_2\neq \h_i}P_t(\h_2|\x_j)\dots\sum_{\h_n\neq \h_i}P_t(\h_n|\x_j)
\]
\[
=\sum_i \sum_{\h_i}(\log P(\x_i,\h_i)) P_t(\h_i|\x_i)\prod_{j\neq
i}\sum_{\h_j}P_t(\h_j|\x_j)
\]
if yes then i'm ok.
}\NDY{Yes except that by $\h_1\neq \h_i$ you don't mean that $\h_1$ and
$\h_i$ have different values, but that you include this term if $i\neq 1$}
but of course $\sum_{\h_j} P_t(\h_j|\x_j)=1$ so that finally
\NDY{Hopefully simpler version}

\begin{align*}
\sum_{\vh} P_t(\vh|\vx)\log P(\vx,\vh)&=\sum_i \sum_{\h_i} (\log
P(\x_i,\h_i)) P_t(\h_i|\x_i)
\\&=\sum_\h \sum_i (\log
P(\x_i,\h)) P_t(\h|\x_i)
\\&=\sum_\h\sum_i (\log P_{\theta_I}(\x_i|\h)+\log Q(\h))P_t(\h|\x_i)
\end{align*}
because $P(\x,\h)=P_{\theta_I}(\x|\h)Q(\h)$. We have to maximize this
quantity over $Q$. The first term does not depend on $Q$ so we only have to
maximize $\sum_\h\sum_i (\log Q(\h))P_t(\h|\x_i)$.

This latter quantity is concave in $Q$, so to find the maximum it is
sufficient to exhibit a point where the derivative w.r.t.\ $Q$ (subject to
the constraint that $Q$ is a probability distribution) vanishes.

Let us compute this derivative.
If we replace $Q$ with $Q+\delta Q$ where $\delta Q$ is
infinitesimal,
the variation of the quantity to be maximized is
\[
\sum_\h \sum_i (\delta \log Q(\h)) P_t(\h|\x_i)=
\sum_\h \frac{\delta Q(\h)}{Q(\h)} \sum_i P_t(\h|\x_i)
\]
\NDL{\emph{proofchecking: } looks right.}
Let us take $Q=(Q_t)\cond_\D$. Since we assumed for simplicity that the
data distribution $\D$ is uniform over the sample this $(Q_t)\cond_\D$ is
\[
Q(\h)=(Q_t)\cond_\D(\h)=\frac1n \sum_i P_t(\h|\x_i)
\]
so that the variation of the quantity to be maximized is
\[
\sum_\h \frac{\delta Q(\h)}{Q(\h)} \sum_i P_t(\h|\x_i)=n\sum_\h \delta
Q(\h)
\]
But since $Q$ and $Q+\delta Q$ are both probability distributions, both
sum to $1$ over $\h$ so that $\sum_\h \delta
Q(\h)=0$. This proves that this choice of $Q$ is an extremum of the
quantity to be maximized.
\NDL{\emph{proofchecking: }Ok, calculus of variations is cool.}

This proves the last statement of the proposition. As mentioned above, it
implies the first by the general properties of EM.
\NDL{\emph{proofchecking: } Ok, good. $(Q_t)\cond_\D$ is an optimum of the EM step.}
Once the first
statement is proven, the best latent marginal $\hat Q_{\theta_I,\D}$ has
to be a fixed point of data incorporation, because otherwise we would get
an even better distribution thus contradicting the definition of the BLM.
\NDL{Agreed.}

The only point left to prove is the equivalence between critical points
of the log-likelihood and fixed points of $Q\mapsto Q\cond_\D$. This is
a simple instance of maximization under constraints, as follows.
Critical points of the data log-likelihood are those for which the
log-likelihood does not change at first order when $Q$ is replaced with
$Q+\delta Q$ for small $\delta Q$. The only constraint on $\delta Q$ is
that $Q+\delta Q$ must still be a probability distribution, so that
$\sum_\h \delta Q(\h)=0$ because both $Q$ and $Q+\delta Q$ sum to $1$.

The first-order variation of
log-likelihood is
\begin{align*}
\delta \sum_i \log P(\x_i)
&=
\delta \sum_i \log\left(
\sum_\h P_{\theta_I}(\x_i|\h)Q(\h)
\right)
\\&=\sum_i \frac{\delta \sum_\h P_{\theta_I}(\x_i|\h)Q(\h)}{\sum_\h
P_{\theta_I}(\x_i,\h)Q(\h)}
\\&=\sum_i \frac{\sum_\h P_{\theta_I}(\x_i|\h) \delta Q(\h)}{P(\x_i)}
\\&=\sum_\h \delta Q(\h) \sum_i \frac{P_{\theta_I}(\x_i|\h)}{P(\x_i)}
\\&=\sum_\h \delta Q(\h) \sum_i \frac{P(\x_i,\h)/Q(\h)}{P(\x_i)}
\\&=\sum_\h \delta Q(\h) \sum_i \frac{P(\h|\x_i)}{Q(\h)}
\end{align*}
This must vanish for any $\delta Q$ such that $\sum_\h \delta Q(\h)=0$.
By elementary linear algebra (or Lagrange multipliers) this occurs if and only if
$\sum_i \frac{P(\h|\x_i)}{Q(\h)}$ does not
depend on $\h$, i.e., if and only if $Q$ satisfies $Q(\h)=C \sum_i
P(\h|\x_i)$. Since $Q$ sums to $1$ one finds $C=\frac1n$. Since all along $P$ is the probability distribution on $\x$
and $\h$ defined by $Q$ and $P_{\theta_I}(\x|\h)$, namely,
$P(\x,\h)=P_{\theta_I}(\x|\h)Q(\h)$, by definition we have
$P(\h|\x)=Q\cond(\h|\x)$ so that the condition $Q(\h)=\frac{1}{n} \sum_i
P(\h|\x_i)$ exactly means
that $Q=Q\cond_\D$, hence the equivalence between critical points of log-likelihood
and fixed points of data incorporation.
\NDL{\emph{proofchecking: } Still a bit unsure about a few details but looks ok. I'll reread it to be sure.}
\end{proof}

\section{Applications and Experiments}
\label{sec:exp}

Given the approach described above, we now consider
several applications for which we evaluate the method empirically.

The intractability of the log-likelihood for deep networks makes direct
comparison of several methods difficult in general. Often the
evaluation is done by using latent variables as features for a
classification task and by direct visual comparison of samples generated
by the model \cite{Larochelle2009,Salakhutdinov2009}.
Instead,
we introduce two new datasets which are simple enough for the true
log-likelihood to be computed explicitly, yet
complex enough to be relevant to deep learning.

We first check that these two datasets are indeed deep.

Then we try to assess the impact of the various approximations from
theory to practice, on the validity of the approach.

We then apply our method to the training of deep
belief networks using properly modified auto-encoders, and show that
the method outperforms current state of the art.

We also explore the use of the BLM upper bound to perform layer-wise
hyper-parameter selection and show that it gives an accurate prediction of
the future log-likelihood of models.

\subsection{Low-Dimensional Deep Datasets}

We now introduce two new deep datasets of low dimension. In order
for those datasets to give a reasonable picture of what happens
in the general case, we first have to make sure that they are relevant to
deep learning, using the following approach:
\begin{enumerate}
\item In the spirit of \cite{Bergstra2012}, we train 1000
RBMs using CD-1 \cite{Hinton2002} on the dataset $\mathcal{D}$, and evaluate the log-likelihood
of a disjoint validation dataset $\mathcal{V}$ under each model.
\item We train 1000 2-layer deep networks using stacked RBMs trained with
CD-1 on $\mathcal{D}$, and evaluate the log-likelihood of $\mathcal{V}$
under each model.
\item We compare the performance of each model at equal number of parameters.
\item If deep networks consistently outperform single RBMs for the same
number of parameters, the dataset is considered to be deep.
\end{enumerate}
The comparison at equal number of parameters is justified by one of the main hypotheses of deep learning,
namely that deep architectures are capable of representing some functions more compactly than shallow architectures \cite{Bengio2007b}.

Hyper-parameters taken into account for hyper-parameter random search are the hidden
layers sizes, CD learning rate and number of CD epochs. The corresponding
priors are given in Table~\ref{tab:hp-priors}. In order not to
give an obvious head start to deep networks, the possible layer
sizes are chosen so that the maximum number of parameters for the single RBM
and the deep network are as close as possible.

\begin{table}
\begin{centering}
\begin{tabular}{|l|l|}
\hline 
Parameter & Prior\tabularnewline
\hline 
\hline 
RBM hidden layer size & 1 to 19\tabularnewline
\hline 
Deep Net hidden layer 1 size & 1 to 16\tabularnewline
\hline 
Deep Net hidden layer 2 size & 1 to 16\tabularnewline
\hline 
inference hidden layer size & 1 to 500\tabularnewline
\hline 
CD learn rate & $\log U(10^{-5},5\times10^{-2})$\tabularnewline
\hline 
BP learn rate & $\log U(10^{-5},5\times10^{-2})$\tabularnewline
\hline 
CD epochs & $20\times(10000/N)$\tabularnewline
\hline 
BP epochs & $20\times(10000/N)$\tabularnewline
\hline 
ANN init $\sigma$ & $U(0,1)$\tabularnewline
\hline 
\end{tabular}
\par\end{centering}

\caption{\label{tab:hp-priors}Search space for hyper-parameters when using
random search for a dataset of size $N$.}
\end{table}

\subsubsection*{\noun{Cmnist} dataset}

The \noun{Cmnist} dataset is a low-dimensional variation on the
\noun{Mnist} dataset \cite{LeCun1998a}, containing 12,000 samples of dimension 100. The
full dataset is split into training, validation and test sets of 4,000
samples each. The dataset is obtained by taking a $10\times10$ image at the
center of each \noun{Mnist} sample and using the values in {[}0,1{]}
as probabilities. The first 10 samples of the \noun{Cmnist} dataset
are shown in Figure~\ref{fig:cmnist-sample}.

\begin{figure}[h]
\begin{centering}
\includegraphics[width=0.5\columnwidth]{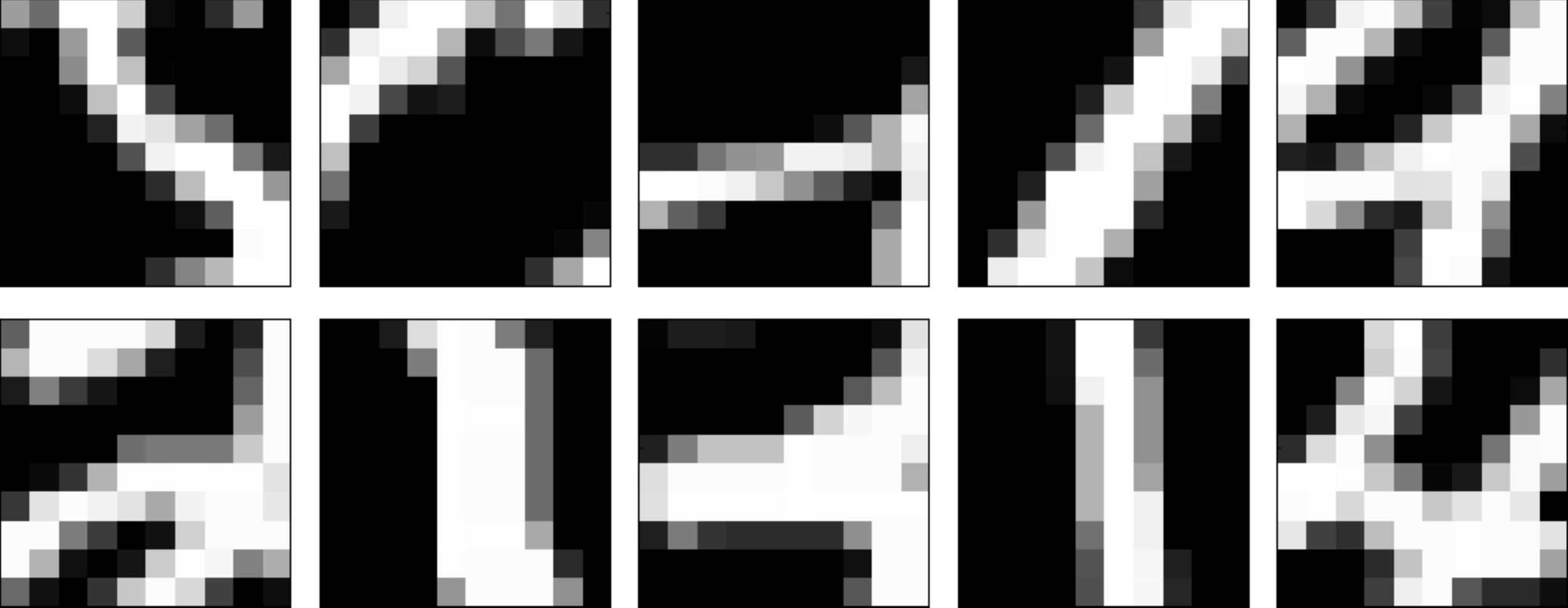}
\par\end{centering}

\caption{\label{fig:cmnist-sample}First 10 samples of the \noun{Cmnist} dataset.}
\end{figure}

We propose two baselines to which to compare the log-likelihood values
of models trained
on the \noun{Cmnist} dataset:
\begin{enumerate}
\item The uniform coding scheme: a model which gives equal probability to
all possible binary $10\times10$ images. The log-likelihood of each
sample is then $-100$ bits, or $-69.31$ nats.
\item The independent Bernoulli model in which each pixel is given an independent
Bernoulli probability. The model is trained on the training set.
The log-likelihood of the validation set is $-67.38$ nats per
sample.
\end{enumerate}
The comparison of the log-likelihood of stacked RBMs with that of
single RBMs is presented in Figure~\ref{fig:cmnist-cmp} and confirms
that the \noun{Cmnist} dataset is deep.

\begin{figure}[h]
\begin{centering}
\includegraphics[width=0.9\columnwidth]{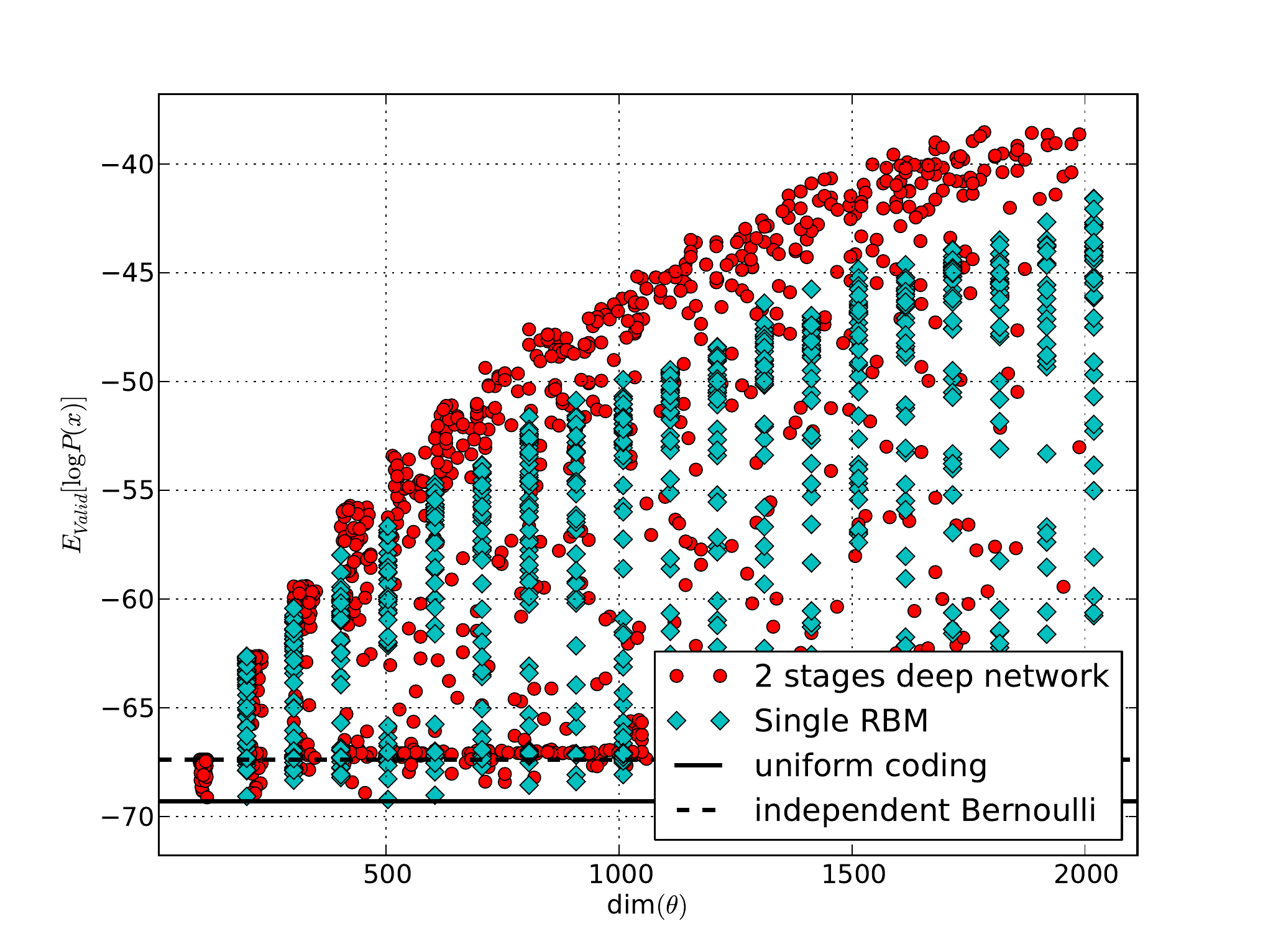}
\par\end{centering}

\caption{\label{fig:cmnist-cmp}Checking that \noun{Cmnist} is deep:
log-likelihood of the validation
dataset $\mathcal{V}$ under RBMs and SRBM deep detworks selected
by hyper-parameter random search, as a function of the number of
parameters $\dim(\theta)$.}
\end{figure}

\subsubsection*{\noun{Tea} dataset}

The \noun{Tea} dataset is based on the idea of learning an invariance for
the amount of liquid in several containers: a teapot and 5 teacups.  It
contains 243 distinct samples which are then distributed into a training,
validation and test set of 81 samples each. The dataset consists of
$10\times 10$
images in which the left part of the image represents a (stylized) teapot
of size $10\times5$. The right part of the image represents 5 teacups of size
$2\times 5$. The liquid is represented by ones and always lies at the bottom of
each container. The total amount of liquid is always equal to the
capacity of the teapot, i.e., there are always 50 ones and 50 zeros in any
given sample. The first 10 samples of the \noun{Tea} dataset are shown in
Figure~\ref{fig:tea-sample}.{\NDY{Invert colors/up and down? I have a lot of difficulty picturing black
air and white liquid.}}

\begin{figure}[h]
\begin{centering}
\includegraphics[width=0.5\columnwidth]{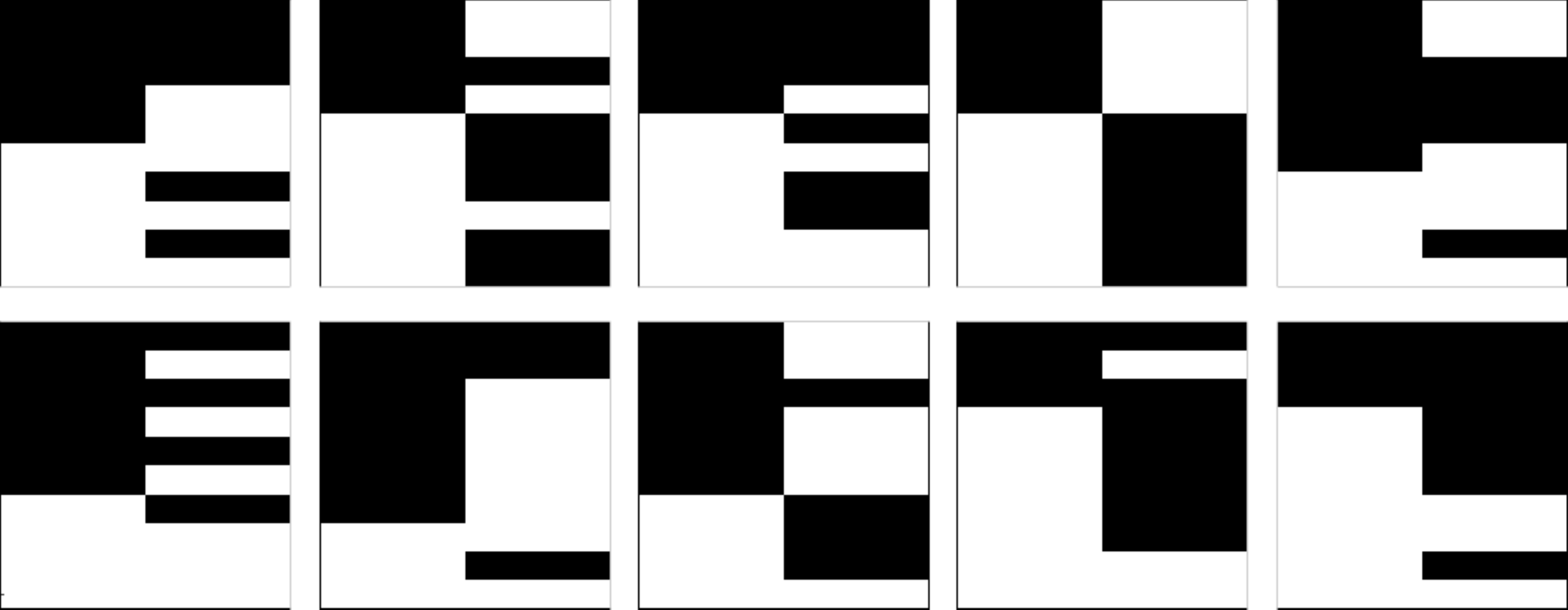}
\par\end{centering}

\caption{\label{fig:tea-sample}First 10 samples of the \noun{Tea}
dataset.}
\end{figure}

In order to better interpret the log-likelihood of models trained
on the \noun{Tea} dataset, we propose 3 baselines:
\begin{enumerate}
\item The uniform coding scheme: the baseline is the same as for the \noun{Cmnist}
dataset: $-69.31$ nats.
\item The independent Bernoulli model, adjusted on the training set. The log-likelihood of the
validation set is $-49.27$ nats per sample.
\item The perfect model in which all $243$ samples of the full dataset
(consituted by concatenation of the training, validation and test
sets) are given the probability $\frac{1}{243}$. The expected log-likelihood
of a sample from the validation dataset is then $\log(\frac{1}{243})=-5.49$
nats.
\end{enumerate}
The comparison of the log-likelihood of stacked RBMs and
single RBMs is presented in Figure~\ref{fig:tea-cmp} and confirms
that the \noun{Tea} dataset is deep.

\begin{figure}[h]
\begin{centering}
\includegraphics[width=0.9\columnwidth]{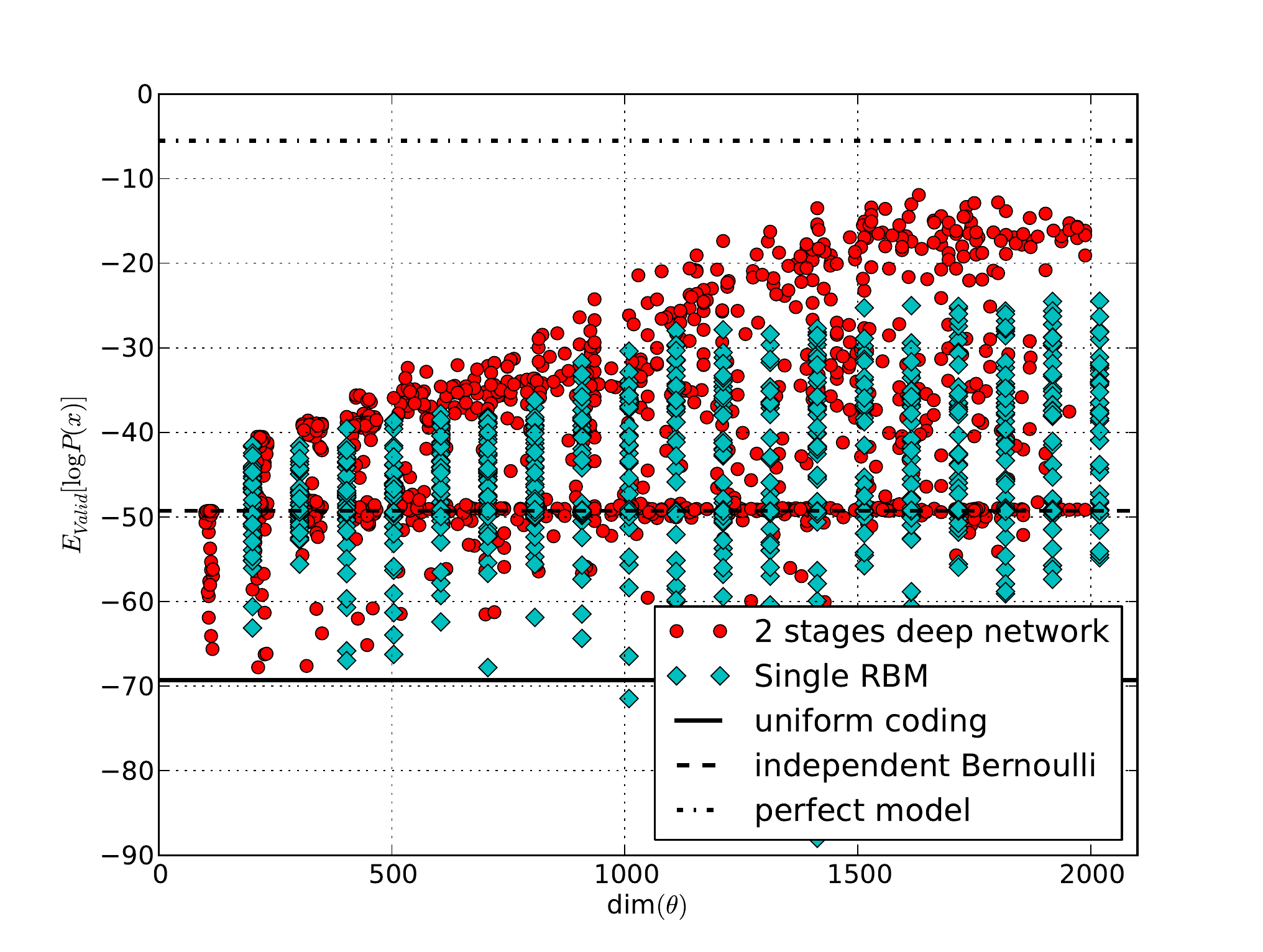}
\par\end{centering}

\caption{\label{fig:tea-cmp}Checking that \noun{Tea} is deep:
log-likelihood of the validation
dataset $\mathcal{V}$ under RBMs and SRBM deep networks selected
by hyper-parameter random search, as a function of the number of parameters $\dim(\theta)$.}
\end{figure}

\subsection{Deep Generative Auto-Encoder Training}

A first application of our approach is the training of a deep \emph{generative}
model using auto-associators. To this end, we propose to train lower
layers using auto-associators and to use an RBM for the generative top
layer model. \NDY{Keep in mind for later the possibility to just use a Bernoulli with one more
layer on top, which has the same number of parameters and is neater.}

We will compare three kinds of deep architectures: standard
auto-encoders with an RBM on top (vanilla AEs), the new \emph{auto-encoders with rich
inference} (AERIes) suggested
by our framework, also with an RBM on top, and, for comparison, stacked
restricted Boltzmann machines (SRBMs). 
All the
models used in this study use the same final generative model class for
$P(\x|\h)$ so that the comparison focuses on the training procedure, on equal ground.
SRBMs are considered the
state of the art \cite{Hinton2006,Bengio2007a}---although performance can
be increased using
richer models \cite{Bengio2012}, our
focus here is not on the model but on the layer-wise training procedure
for a given model class.

In ideal circumstances, we would have compared the log-likelihood
obtained for each training algorithm with the optimum of a deep learning
procedure such as the full gradient ascent procedure
(Section~\ref{sec:Layer-wise-deep-learning}). Instead, since this
ideal deep learning procedure is intractable, SRBMs serve as a reference.

The new AERIes are auto-encoders modified after the following remark: 
the complexity of the inference model used for $q(\mathbf{h}|\mathbf{x})$ can
be increased safely without risking overfit and loss of generalization
power,
because $q$ is
not part of the final generative model, and is used only as a tool for
optimization of the generative model parameters. This would suggest that the
complexity of $q$ could be greatly increased with only positive
consequences on the performance of the model.

AERIes exploit this possibility by having, in each layer, a
modified auto-associator with
two hidden layers instead of one: $\x\to \h'\to \h \to \x$. The
generative part $P_{\theta_{I}}(\mathbf{x}|\mathbf{h})$ will be
equivalent to that of a regular auto-associator, but the inference part
$q(\h|\x)$
will have greater representational power because it includes the hidden
layer $\h'$ (see Figure~\ref{fig:dgaa+-training}).

We will also use the more usual auto-encoders composed of
auto-associators with one hidden layer and tied weights, commonly
encountered in the literature (vanilla AE).

For all models, the deep architecture will be of depth $2$. The stacked
RBMs will be made of two ordinary RBMs.
For AERIes and vanilla AEs, 
the lower part is made of a single auto-associator (modified for AERies), and the
generative top part is an RBM. (Thus they have one layer less than
depicted for the sake of generality in Figures~\ref{fig:dgaa-training}
and~\ref{fig:dgaa+-training}.) For AERIes and vanilla AEs
the lower part of the model is
trained using the usual backpropagation algorithm with cross-entropy
loss, which performs gradient ascent for the probability of \eqref{eq:auto-associator-proba}.
The top RBM is then trained to maximize \eqref{eq:toptraining}.

The competitiveness of each model will be evaluated through a comparison
in log-likelihood over a validation set distinct from the training set.
Comparisons are made for a given identical number of parameters of
the generative model%
\footnote{Because we only consider the generative models obtained, $q$
is never taken into account in the number of parameters of an
auto-encoder or SRBM.
However, the parameters of the top RBM are taken into account as they
are a necessary part of the generative model.%
}. Each model will be given equal chance to find a good optimum
in terms of the number of evaluations in a hyper-parameter selection
procedure by random search.

When implementing the training procedure proposed in
Section~\ref{sec:Layer-wise-deep-learning}, several approximations are
needed.
An important one, compared to Theorem~\ref{thm:main}, is that the distribution
$q(\mathbf{h}|\mathbf{x})$ will not really be
trained over all possible conditional distributions for $\h$ knowing
$\x$. Next, training of the upper layers will of course fail to
reproduce the BLM perfectly. Moreover, auto-associators use an
$\x=\tilde\x$ approximation, cf.\ \eqref{eq:auto-associator-proba}. We will study the effect of these
approximations.

Let us now provide more details for each model.

\paragraph{Stacked RBMs.}

For our comparisons, 1000 stacked RBMs were trained using the
procedure from \cite{Hinton2006}. We used random search on the hyper-parameters,
which are: the sizes of the hidden layers, the CD learning
rate, and the number of CD epochs.

\paragraph{Vanilla auto-encoders.}

The general training algorithm for vanilla auto-encoders is depicted in Figure~\ref{fig:dgaa-training}.
First an auto-associator is trained to maximize the adaptation of
the BLM upper bound for auto-associators presented in \eqref{eq:auto-associator-proba}.
The maximization procedure itself is done with the backpropagation
algorithm and cross-entropy loss. The inference weights are tied to
the generative weights so that
$\mathbf{W}_\mathrm{gen}=\mathbf{W}_\mathrm{inf}^{\top}$
as is often the case in practice. An ordinary RBM is used as a generative
model on the top layer.

\begin{figure}
\begin{centering}
\includegraphics[width=0.8\columnwidth]{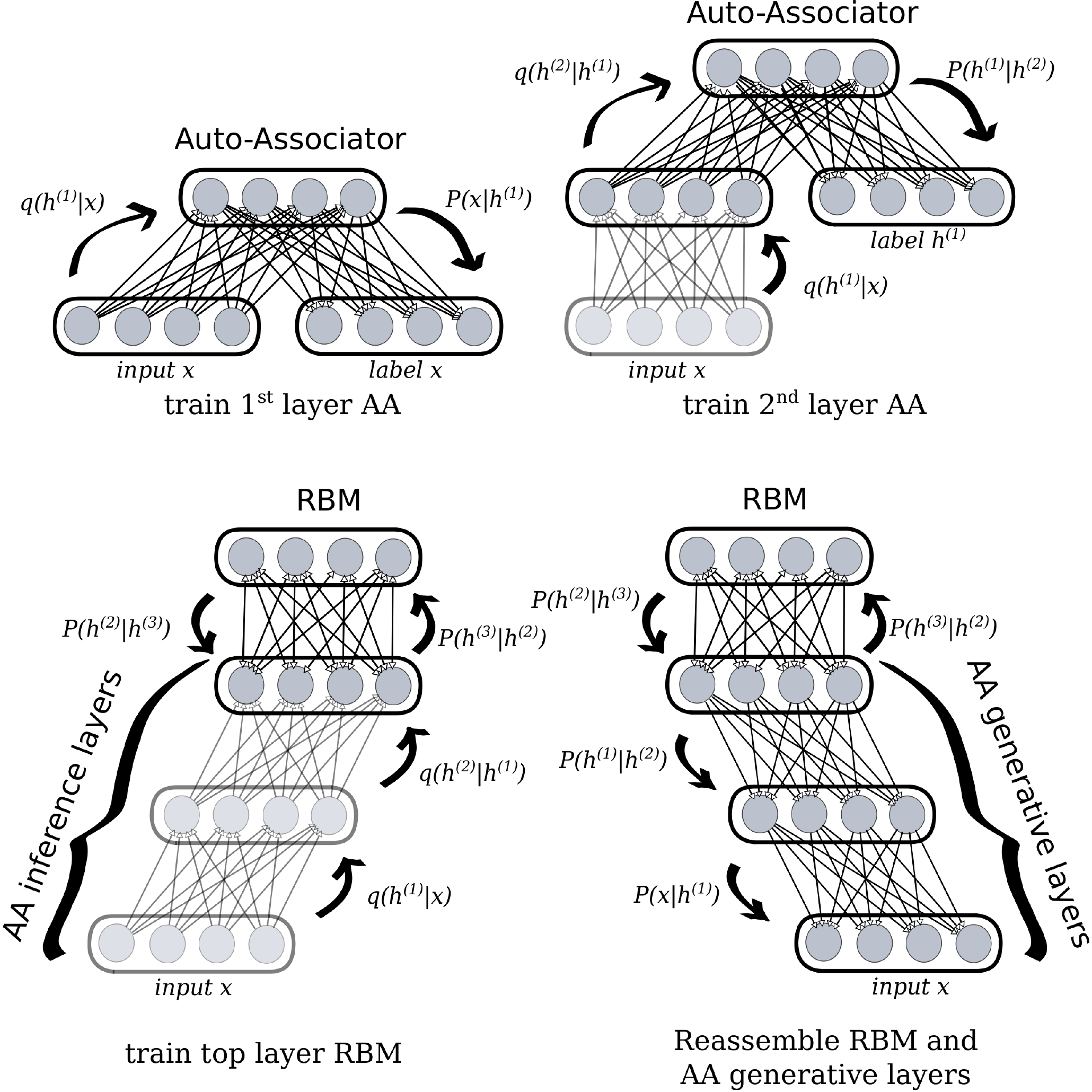}
\par\end{centering}

\caption{\label{fig:dgaa-training}Deep generative auto-encoder training scheme. }
\end{figure}

1000 deep generative auto-encoders
were trained
using random search on the hyper-parameters. Because deep generative
auto-encoders use an
RBM as the top layer, they use the same hyper-parameters as stacked RBMs,
but also backpropagation (BP) learning rate, BP epochs, and ANN init $\sigma$ (i.e. the standard deviation of the gaussian used during
initialization).

\paragraph{Auto-Encoders with Rich Inference (AERIes).}

The model and training scheme for AERIes are represented
in Figure~\ref{fig:dgaa+-training}.  Just as for vanilla auto-encoders,
we use the backpropagation algorithm and cross-entropy loss to maximize
the auto-encoder version \eqref{eq:auto-associator-proba} of the BLM upper bound
on the training set. No
weights are tied, of course, as this does not make sense for an
auto-associator with  different models for
$P(\x|\h)$ and $q(\h|\x)$. The top RBM is trained
afterwards. Hyper-parameters are the same as above, with in addition the size of the new hidden layer $\h'$.

\begin{figure}
\begin{centering}
\includegraphics[width=0.8\columnwidth]{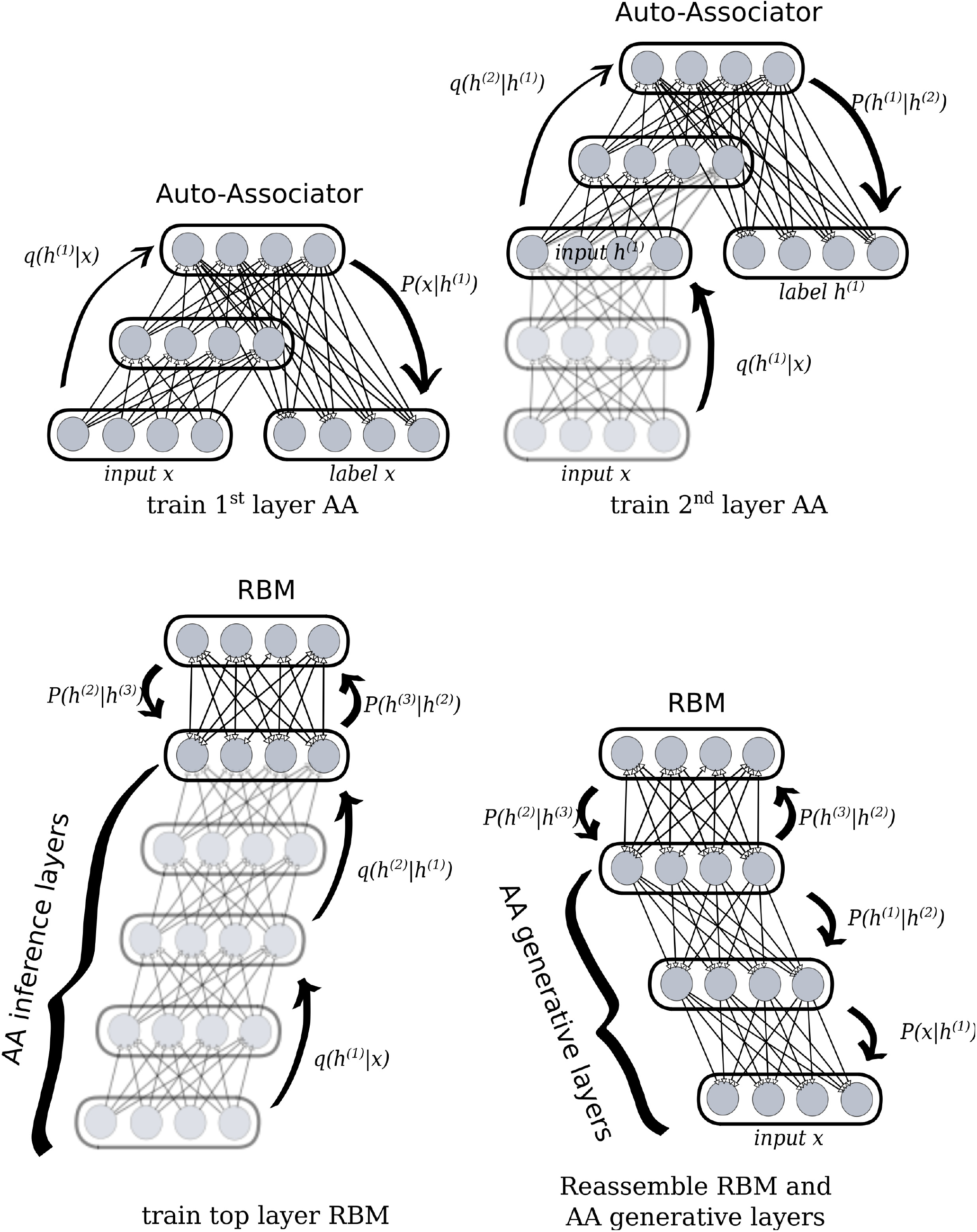}
\par\end{centering}

\caption{\label{fig:dgaa+-training}Deep generative modified auto-encoder
(AERI) training
scheme. }
\end{figure}

\subsubsection*{Results}

The results of the above comparisons on the \noun{Tea} and \noun{Cmnist} \emph{validation} datasets are given in Figures~\ref{fig:tea-cmp-rbm-aes} and \ref{fig:cmnist-cmp-rbm-aes}.
For better readability, the Pareto front%
\footnote{The Pareto front is composed of all models which are not subsumed
by other models according to the number of parameters and the expected
log-likelihood. A model is said to be subsumed by another if it has
strictly more parameters and a worse likelihood.%
}
for each model is given in
Figures~\ref{fig:tea-cmp-rbm-aes-pareto} and \ref{fig:cmnist-cmp-rbm-aes-pareto}.

\NDY{Performance on the validation dataset is better described as
``average''
rather than ``expected'', as we do not describe the validation dataset by a
probability distribution...}

\begin{figure}
\begin{centering}
\includegraphics[width=0.9\columnwidth]{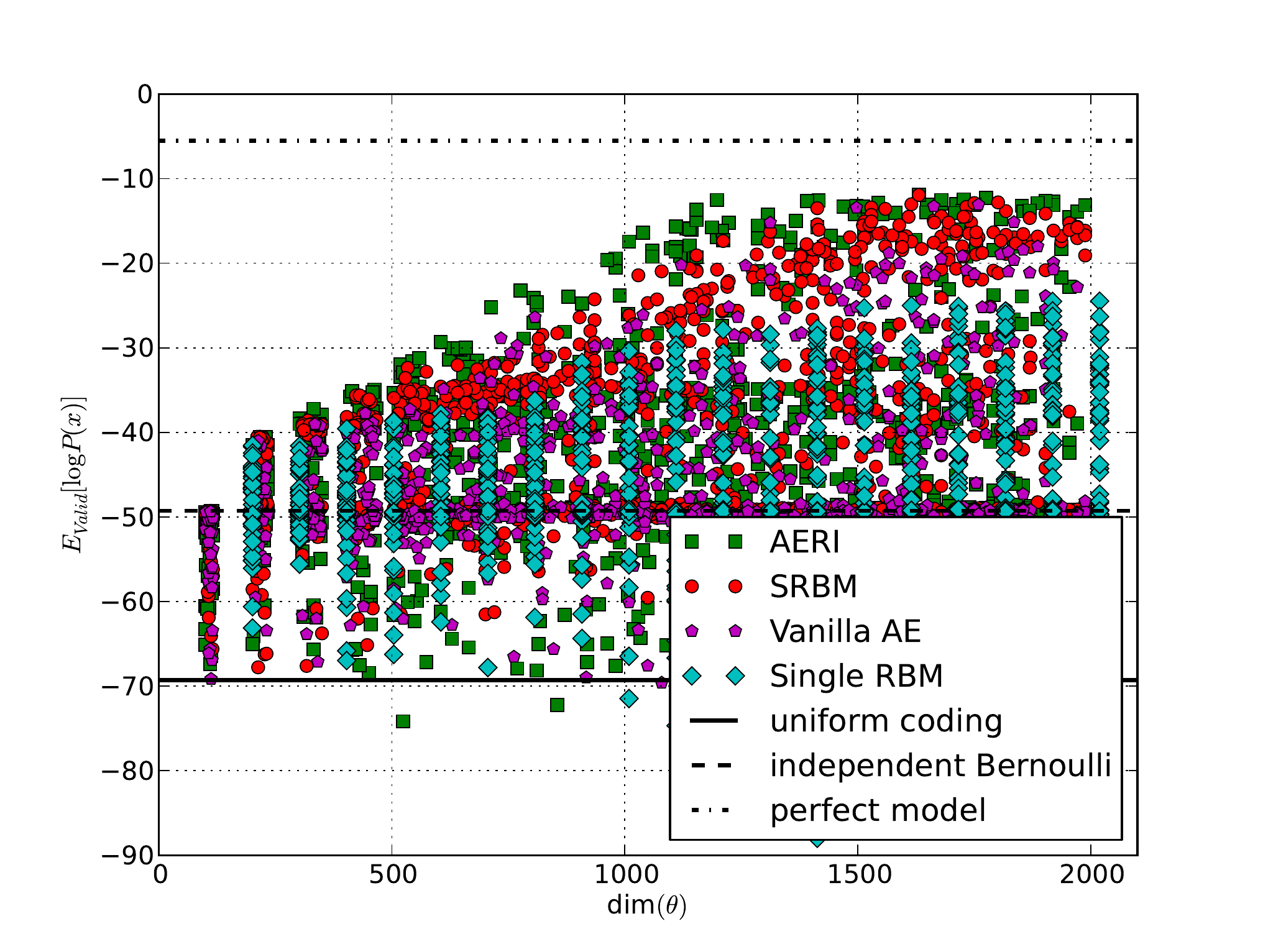}
\par\end{centering}

\caption{\label{fig:tea-cmp-rbm-aes}Comparison of the average validation
log-likelihood for SRBMs, vanilla AE, and AERIes on the \noun{Tea} dataset.}
\end{figure}

\begin{figure}
\begin{centering}
\includegraphics[width=0.9\columnwidth]{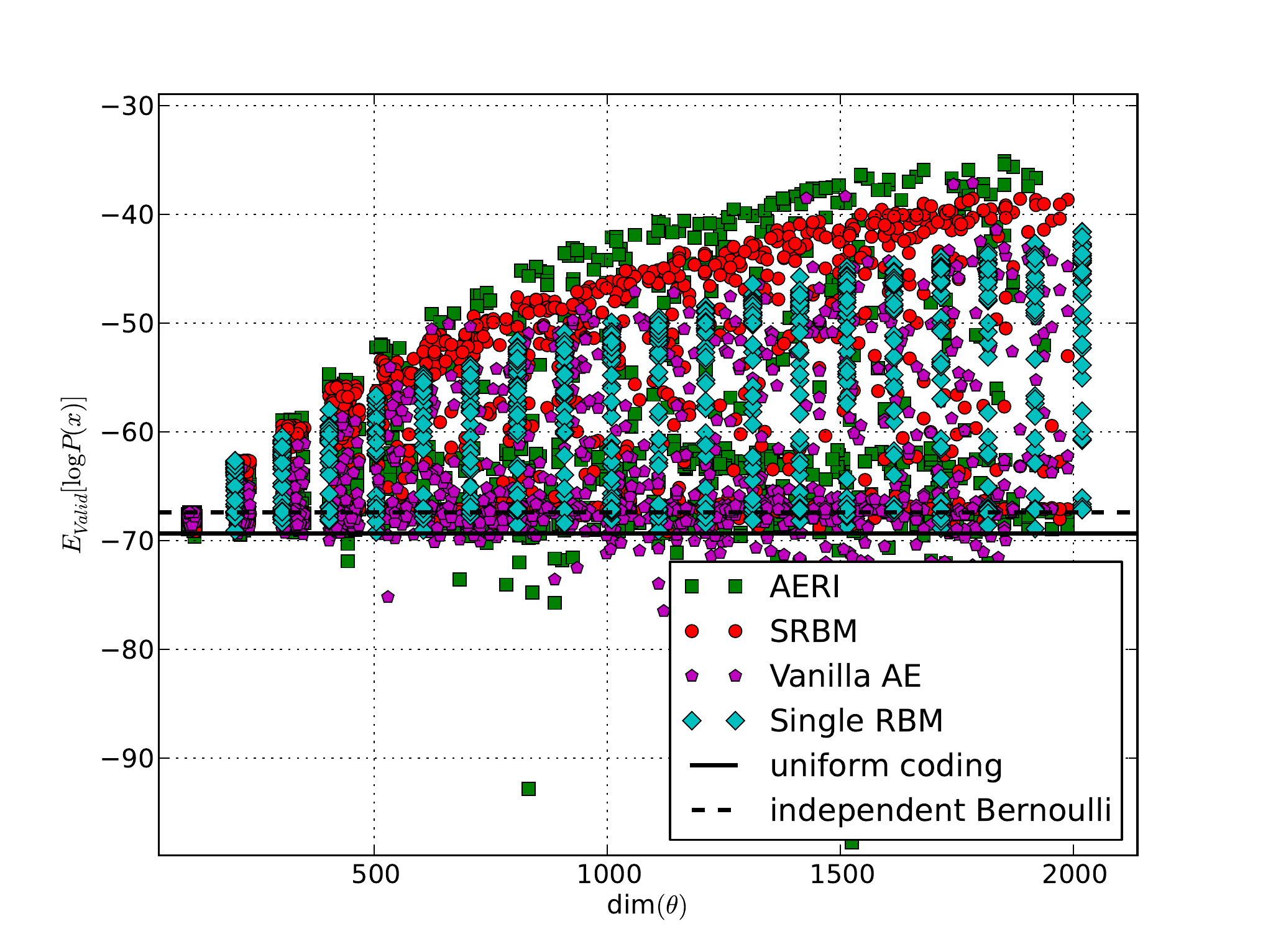}
\par\end{centering}

\caption{\label{fig:cmnist-cmp-rbm-aes}Comparison of the average validation
log-likelihood for SRBMs, vanilla AE, and AERIes on the \noun{Cmnist} dataset.}
\end{figure}

\begin{figure}
\begin{centering}
\includegraphics[width=0.9\columnwidth]{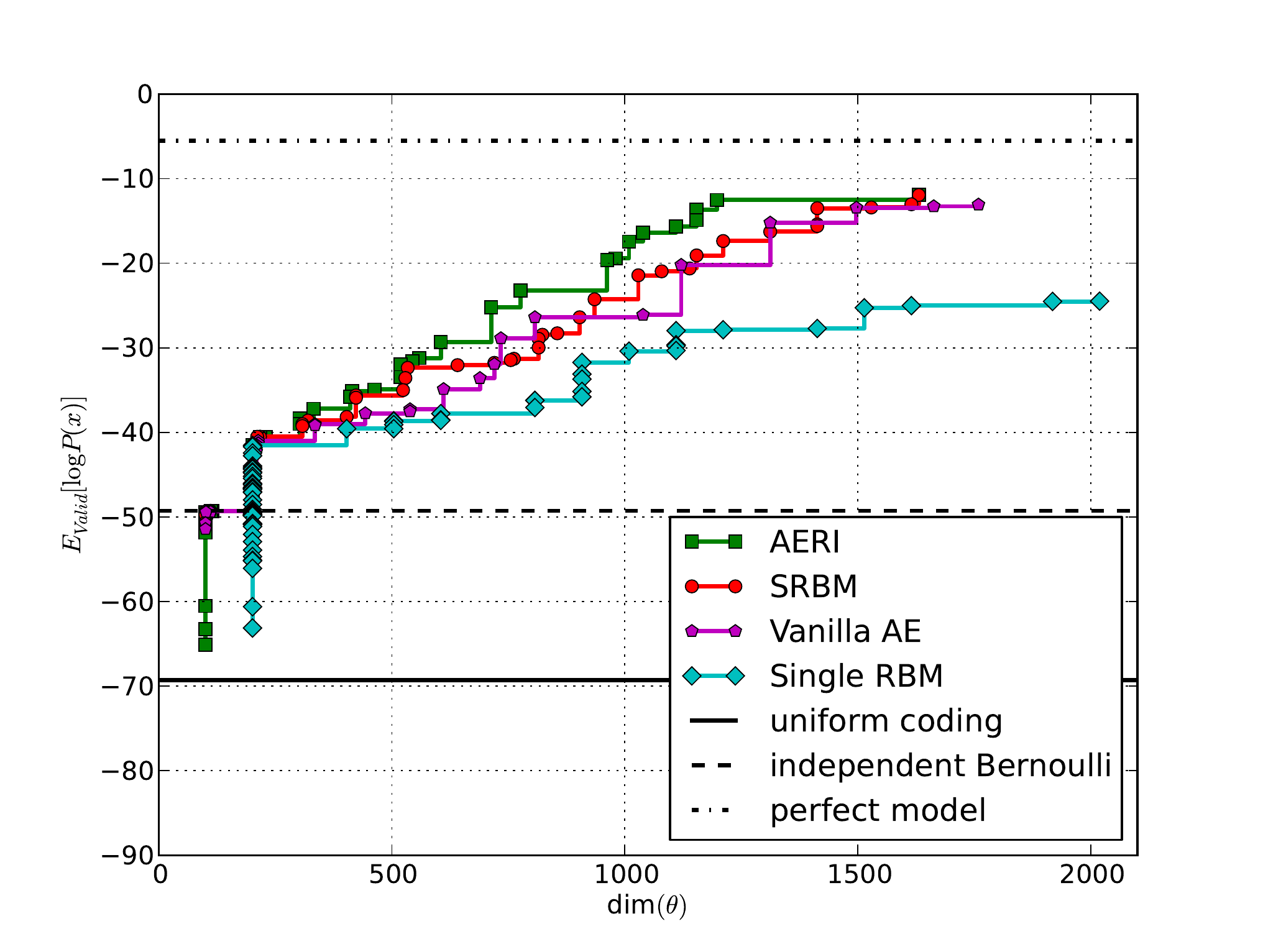}
\par\end{centering}

\caption{\label{fig:tea-cmp-rbm-aes-pareto}Pareto fronts for
the average validation log-likelihood and number of parameters for
SRBMs, deep generative auto-encoders, and modified deep generative auto-encoders
on the \noun{Tea} dataset.}
\end{figure}

\begin{figure}
\begin{centering}
\includegraphics[width=0.9\columnwidth]{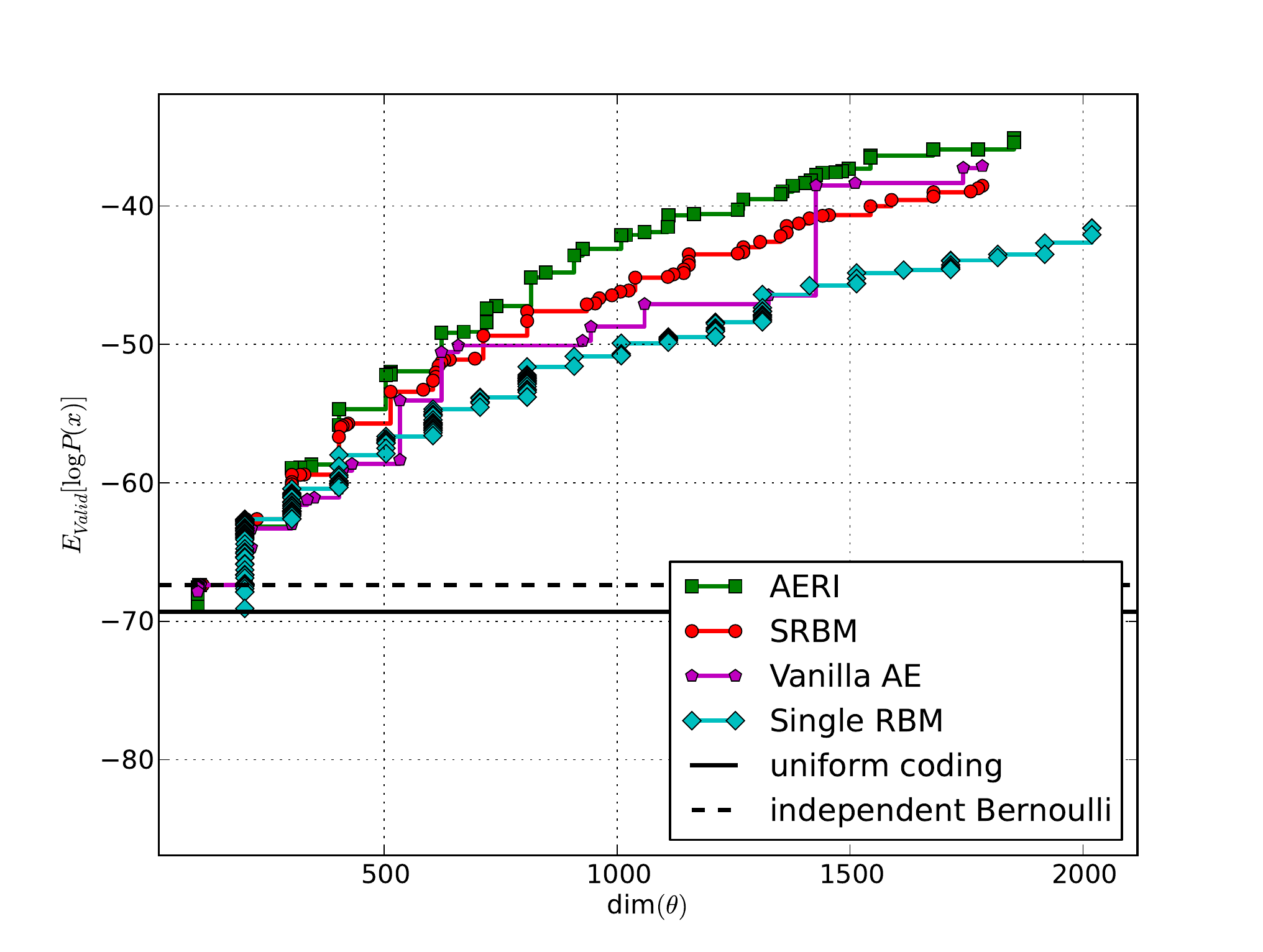}
\par\end{centering}
\caption{\label{fig:cmnist-cmp-rbm-aes-pareto}Pareto fronts
for the average validation log-likelihood and number of parameters
for SRBMs, deep generative auto-encoders, and modified deep generative
auto-encoders on the \noun{Cmnist} dataset.}
\end{figure}

As expected, all models perform better than the baseline independent
Bernoulli model but have a lower likelihood than the perfect
model\footnote{Note that some instances are outperformed by the uniform coding scheme,
which may seem surprising.
Because we are considering the average log-likelihood on a
validation set, if even one sample of the validation set happens to be given
a low probability by the model, the average log-likelihood will be
arbitrarily low. In fact, because of roundoff errors in the computation
of the log-likelihood, a few models have a measured performance of
$-\infty$. This does not affect the comparison of the models as it only
affects instances for which performance is already very low.}.
Also, SRBMs, vanilla AEs and AERIes perform better than a
single RBM, which can be seen as further evidence that the
\noun{Tea} and \noun{Cmnist} are deep datasets.

Among deep models, vanilla auto-encoders achieve the lowest
performance, but outperform single RBMs significantly, which
validates them not only as generative models but also as \emph{deep}
generative models. Compared to SRBMs, vanilla auto-encoders
achieve almost identical performance but the algorithm clearly
suffers from local optima: most instances perform poorly and only a
handful achieve performance comparable to that of SRBMs or AERIes.

As for the auto-encoders with rich inference (AERIes), they are able to outperform not only single RBMs
and vanilla auto-encoders, but also stacked RBMs, and do so consistently. This validates not only
the general deep learning procedure of
Section~\ref{sec:Layer-wise-deep-learning}, but arguably also the
understanding of auto-encoders in this framework.

The results confirm that a more universal model for $q$ can significantly
improve the performance of a model, as is clear from comparing the
vanilla and rich-inference auto-encoders. Let us insist that the rich-inference
auto-encoders and vanilla auto-encoders optimize over exactly the
\emph{same} set of generative models with the same structure, and thus
are facing exactly the same optimization problem~\eqref{eq:objective}.
Clearly the modified training procedure yields improved values of the
generative parameter $\theta$.

\subsection{Layer-Wise Evaluation of Deep Belief Networks}

As seen in section \ref{sec:Layer-wise-deep-learning}, the BLM upper
bound $ U_\D(\theta_I) $ is the least upper bound of the log-likelihood
of deep generative models using some given $\theta_{I}$ in the lower part
of the model. This raises the question of whether it is a good indicator
of the final performance of $\theta_I$.

In this setting, there are
a few approximations w.r.t.~\eqref{eq:bottomtraining} and
\eqref{eq:toptraining} that need to be discussed.
Another point is the intractability of the BLM upper bound for models
with many hidden variables, which leads us to propose and test an estimator in
Section~\ref{sec:blmapprox},
though the experiments considered here
were small enough not to need this unless otherwise specified.

We now look, in turn, at how the BLM upper bound can be applied to
log-likelihood estimation, and to hyper-parameter selection---which can
be considered part of the training procedure. We first discuss various
possible effects, before measuring them empirically.

\subsubsection{Approximations in the BLM upper bound}

Consider the maximization of~\eqref{eq:defU}. In practice, we do
not perform a specific maximization over $q$ to obtain the BLM as in
\eqref{eq:defU}, but rely on the training procedure of $\theta_I$ to maximize it. Thus
the $q$ resulting from a training procedure is generally not the
globally optimal $\hat{q}$ from Theorem~\ref{thm:main}. In the
experiments we of course use the
BLM upper bound with the value of $q$ resulting from the actual training.

\begin{defi}
For $\theta_I$ and $q$ resulting from the training of a deep generative model, let 
\begin{eqnarray}
\hat\U_{\D,q}(\theta_I)
\deq
\mathbb{E}_{\x\sim
P_{\mathcal{D}}}\left[
\log \sum_\h P_{\theta_I}(\x|\h)  q_\D(\h)\right]
\label{eq:defhatU}
\end{eqnarray}
be the \emph{empirical BLM upper bound}.
\end{defi}

This definition makes no assumption about how $\theta_I$
and ${q}$ in the first layer have been trained, and can be applied to
any layer-wise training procedure, such as SRBMs.

Ideally, this quantity should give us an idea of the final performance of
the deep architecture when we use $\theta_I$ on the bottom layer.
But there are several discrepancies between these BLM estimates and final
performance.

A first question is the validity of the approximation~\eqref{eq:defhatU}. The BLM
upper bound $ \U_{\D}(\theta_I)$ is obtained by maximization over all
possible $q$ which is of course untractable. The learned inference distribution
${q}$ used in practice is only an approximation for two reasons:
first, because the model for $q$ may not cover all possible conditional
distributions $q(\h|\x)$, and, second,
because the training of $q$ can
be imperfect. In effect $\hat\U_{\D,{q}}(\theta_I)$ is only a
lower bound of the BLM upper bound : $\hat\U_{\D,{q}}(\theta_I)
\leq \U_{\D}(\theta_I)$.

Second, we can question the relationship between the (un-approximated)
BLM upper bound \eqref{eq:defU} and the final log-likelihood of the model.
The BLM bound is optimistic, and tight only when the upper part of the model manages to
reproduce the BLM perfectly.
We should check how tight it
is in practical applications when the upper layer model for $P(\mathbf{h})$
is imperfect. 

In addition, as for any estimate from a training set, final performance
on validation and test sets might be different. Performance of a
model on the validation set is generally lower than on the
training set. But on the other hand, in our situation there is a specific
\emph{regularizing effect of
imperfect training of the top layers}. Indeed
the BLM refers to a universal optimization over all
possible distributions on $\h$ and might therefore overfit more, hugging the training set
too closely. Thus if we did manage to reproduce the BLM perfectly on the
training set, it could well decrease performance on the validation set. On
the other hand,
training the top layers to approximate the BLM within a model class
$P_{\theta_J}$  introduces further regularization and
could well yield higher final
performance on the validation set than if the exact BLM distribution had been
used.

This latter regularization effect is relevant if we are to use the
BLM upper bound for hyper-parameter selection, a scenario in which
regularization is expected to play an important role.

We can therefore expect:
\begin{enumerate}
\item That the ideal BLM upper bound, being by definition optimistic, can be higher that
the final likelihood
when the model obtained for $P(\mathbf{h})$ is not perfect.
\item That the empirical bound obtained by using a given conditional
distribution $q$ will be lower than
the ideal BLM upper bound
either when $q$ belongs to a restricted class, or when $q$
is poorly trained.
\item That the ideal BLM upper bound on the training set may be either higher or lower than
actual performance on a validation set, because of the regularization
effect of imperfect top layer training.
\end{enumerate}

All in all, the relationship between the empirical BLM upper bound used in training,
and the final log-likelihood on real data, results from several effects
going in both directions. This might affect whether the empirical BLM
upper bound can really be used to predict the future performance of a
given bottom layer setting.

\subsubsection{A method for single-layer evaluation
and
layer-wise hyper-parameter selection}

In the context of deep architectures, hyper-parameter selection is a difficult problem.
It can involve as much as 50 hyper-parameters, some of them only relevant
conditionally to others \cite{Bergstra2011,Bergstra2012}. To make matters worse, evaluating the generative performance of such models is often intractable.
The evaluation is usually done w.r.t.\ classification performance as in
\cite{Larochelle2009,Bergstra2011,Bergstra2012}, sometimes complemented by a visual comparison of samples from the model \cite{Hinton2006,Salakhutdinov2009}. In some rare instances, a variational approximation of the log-likelihood is considered \cite{Salakhutdinov2008,Salakhutdinov2009}.

These methods only consider evaluating the models after all layers have
been fully trained. However, since the training of deep architectures is
done in a layer-wise fashion, with some criterion greedily maximized at
each step, it would seem reasonable to perform a layer-wise evaluation.
This would have the advantage of reducing the size of the
hyper-parameter search space
from exponential to linear in the number of
layers.

We propose to first evaluate the performance of the lower layer,
after it has been trained, according
to the BLM upper bound \eqref{eq:defhatU} (or an approximation thereof) on the validation
dataset $\D_{\mathrm{valid}}$\NDY{Otherwise larger hidden layers learning the identity will
always be the winners.}\NDL{ True...}. The measure of performance obtained can then
be used as part of a larger hyper-parameter selection algorithm such as
\cite{Bergstra2012,Bergstra2011}.  This results in further optimization
of \eqref{eq:bottomtraining} over the hyper-parameter space and is
therefore justified by Theorem~\ref{thm:main}.

Evaluating the top layer is less problematic: by definition, the top layer is always a ``shallow'' model for which the true likelihood becomes more easily tractable. For instance, although
RBMs are well known to have an intractable partition function which prevents their evaluation, several methods
are able to compute close approximations to the true likelihood (such as Annealed Importance Sampling
\cite{Neal1998,Salakhutdinov2008}). The dataset to be evaluated with this procedure will have to be a sample of
$\sum_\x q(\h|\x)P_\D(\x)$.

In summary, the evaluation of a two-layer generative model can be done in a layer-wise manner:

\begin{enumerate}
\item Perform hyper-parameter selection on the lower layer using
$\hat\U_{\theta_I}(\D)$ as a performance measure (preferably on a
validation rather than training dataset, see below), and keep only the best
possible lower layers according to this criterion.
\item Perform hyper-parameter selection on the upper layer by evaluating
the true likelihood of validation data samples transformed by the inference
distribution, under the model of the top layer\footnote{This could lead to a stopping
criterion when training a model with arbitrarily many layers: for the upper
layer, compare the likelihood of the best upper-model with the
BLM of the best possible next layer. If the BLM of the next
layer is not significatively higher than the likelihood of the
upper-model, then we do not add another layer as it would not
help to achieve better performance.
}.
\end{enumerate}

Hyper-parameter selection was not used in our experiments,
where we simply used hyper-parameter random search. (This has allowed, in
particular, to check the robustness of the models, as AERIes have been
found to perform better than vanilla AEs on many more instances over
hyper-parameter space.)

As mentioned earlier, in the context of representation learning the top
layer is irrelevant because the objective is not to train a generative
model but to get a better representation of the data. With the assumption
that good latent variables make good representations, this suggests that
the BLM upper bound can be used directly to select the best possible
lower layers.

\subsubsection{Testing the BLM and its approximations}

We now present a series of tests to check whether the selection of lower
layers with higher values of the BLM actually results in higher
log-likelihood for the final deep generative models, and to assess the
quantitative importance of each of the BLM approximations discussed
earlier.

For each training algorithm (SRBMs, RBMs, AEs, AERIes), the comparison is
done using 1000 models trained with hyper-parameters selected through
random search as before. The empirical BLM upper bound is computed using
\eqref{eq:defhatU} above.

\paragraph{Training BLM upper bound vs training log-likelihood.}

We first compare the value of the empirical BLM upper bound
$\hat\U_{\theta_I}(\D_\mathrm{train})$ over the training set, with the actual log-likelihood of the
trained model on the training set. This is an evaluation of how
optimistic the BLM is for a given dataset, by checking how closely the
training of the upper layers manages to match the target BLM distribution
on $\h$. This is also the occasion to check the effect of
using the ${q}(\h|\x)$ resulting from actual learning,
instead of the best $q$ in all possible conditional distributions.

In addition, as discussed below, this comparison can be used as a
criterion to determine whether more layers should be added to the model.

The results are given in Figures~\ref{fig:tea-training-upper-bound}
and~\ref{fig:cmnist-training-upper-bound} for SRBMs, and~\ref{fig:tea-ae3-training-upper-bound}
and~\ref{fig:cmnist-ae3-training-upper-bound} for AERIes.
We see that the empirical BLM upper bound
\eqref{eq:defhatU} is a good predictor of the future
log-likelihood of the full model on the \emph{training} set. This shows that the approximations w.r.t.\ the
optimality of the top layer and the universality of $q$ can be dealt with in practice.

For AERIes, a few models with low performance have a poor estimation of
the BLM upper bound (estimated to be lower than the actual likelihood),
presumably because of a bad approximation in the learning of $q$. This
will not affect model selection procedures as it only concerns models
with very low performance, which are to be discarded.

\begin{figure}
\begin{centering}
\includegraphics[clip,width=0.8\columnwidth]{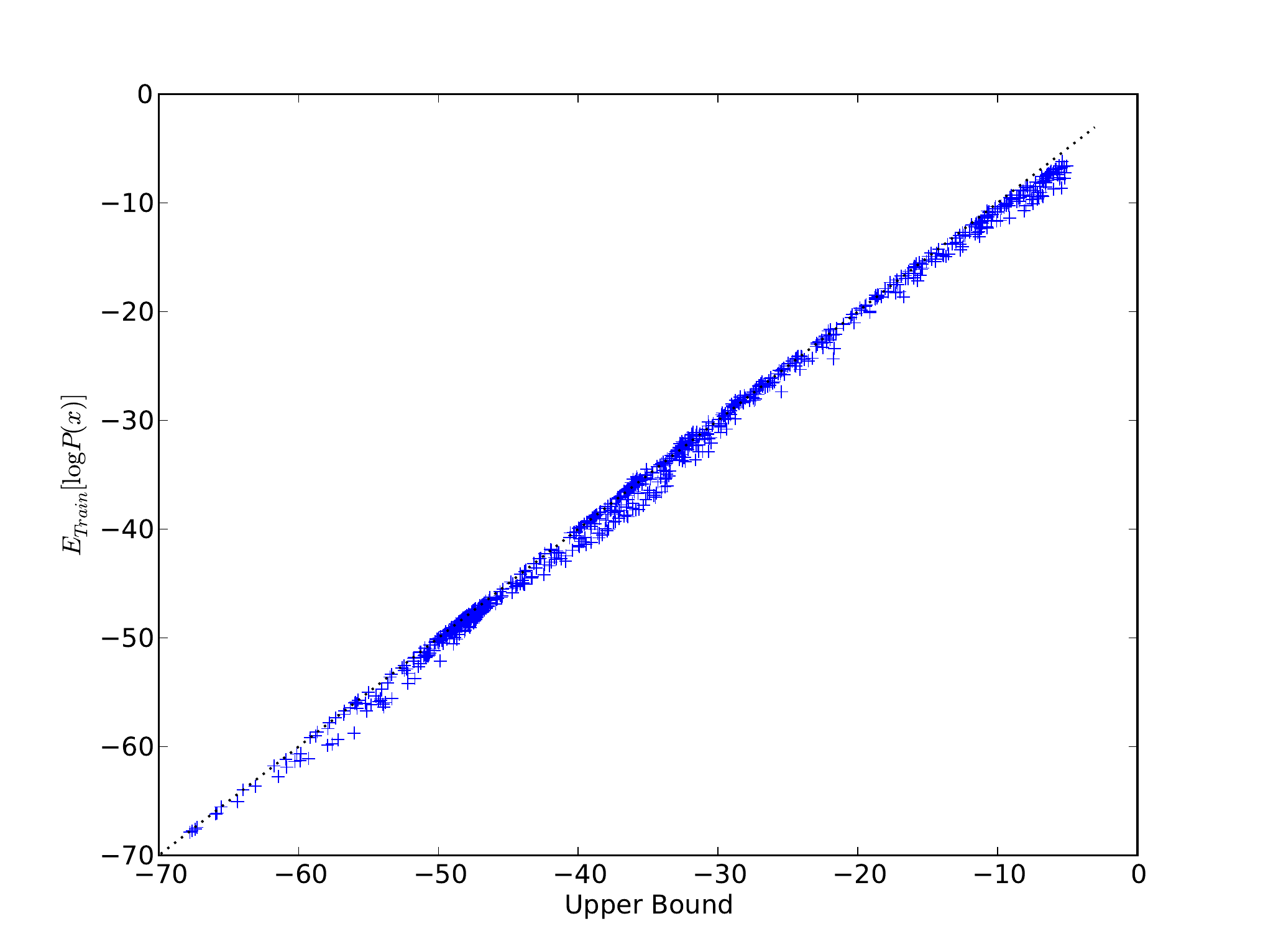}
\par\end{centering}

\caption{\label{fig:tea-training-upper-bound}Comparison of the BLM upper bound
on the first layer and the final log-likelihood on the \noun{Tea}
training dataset, for 1000 2-layer
SRBMs}
\end{figure}

\begin{figure}
\begin{centering}
\includegraphics[clip,width=0.8\columnwidth]{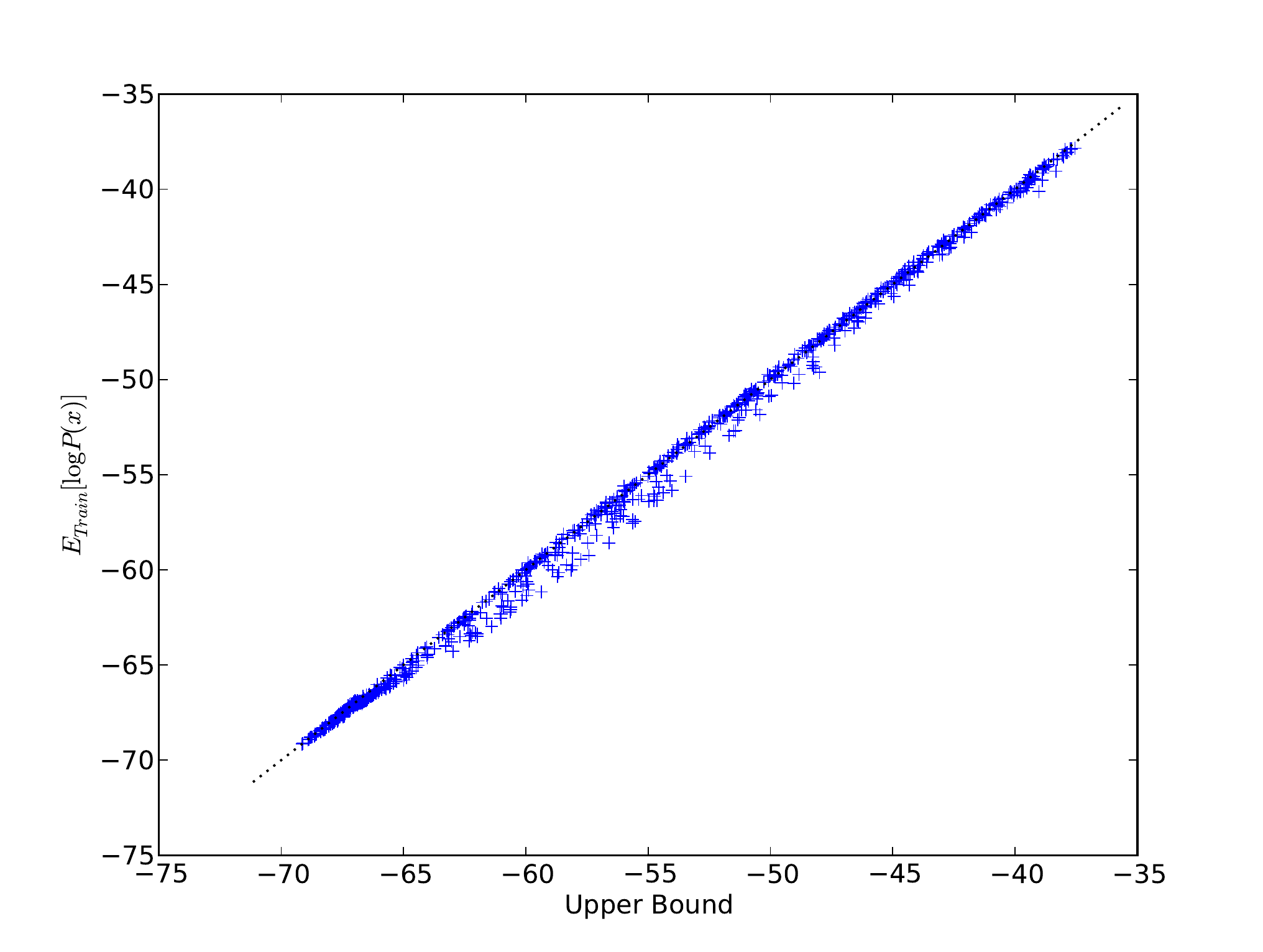}
\par\end{centering}

\caption{\label{fig:cmnist-training-upper-bound}Comparison of the BLM upper bound
on the first layer and the final log-likelihood on the \noun{Cmnist} training dataset, for 1000 2-layer
SRBMs}
\end{figure}

\begin{figure}
\begin{centering}
\includegraphics[clip,width=0.8\columnwidth]{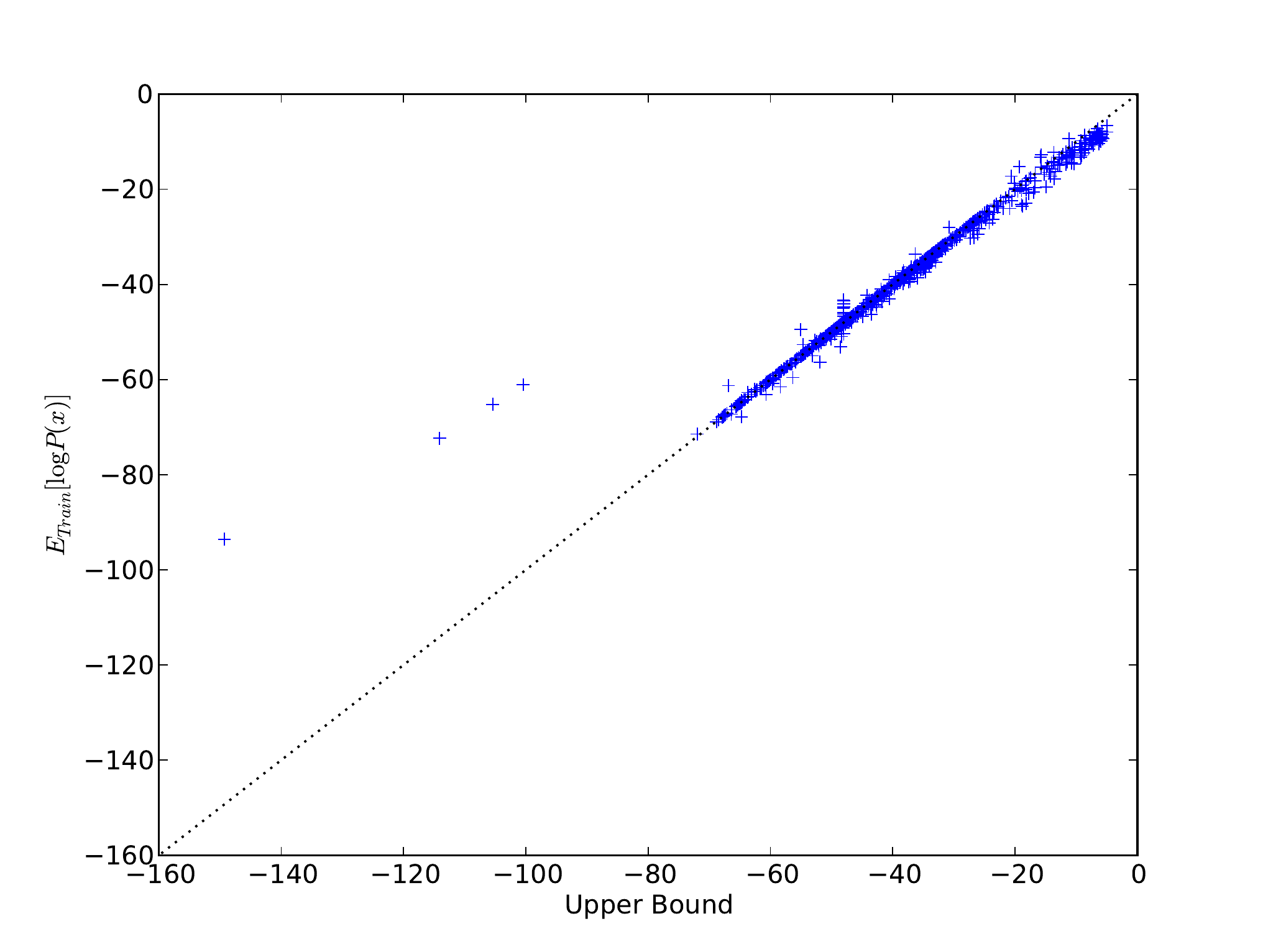}
\par\end{centering}

\caption{\label{fig:tea-ae3-training-upper-bound}Comparison of the BLM upper bound
on the first layer and the final log-likelihood on the \noun{Tea}
training dataset, for 1000 2-layer
AERIes}
\end{figure}

\begin{figure}
\begin{centering}
\includegraphics[clip,width=0.8\columnwidth]{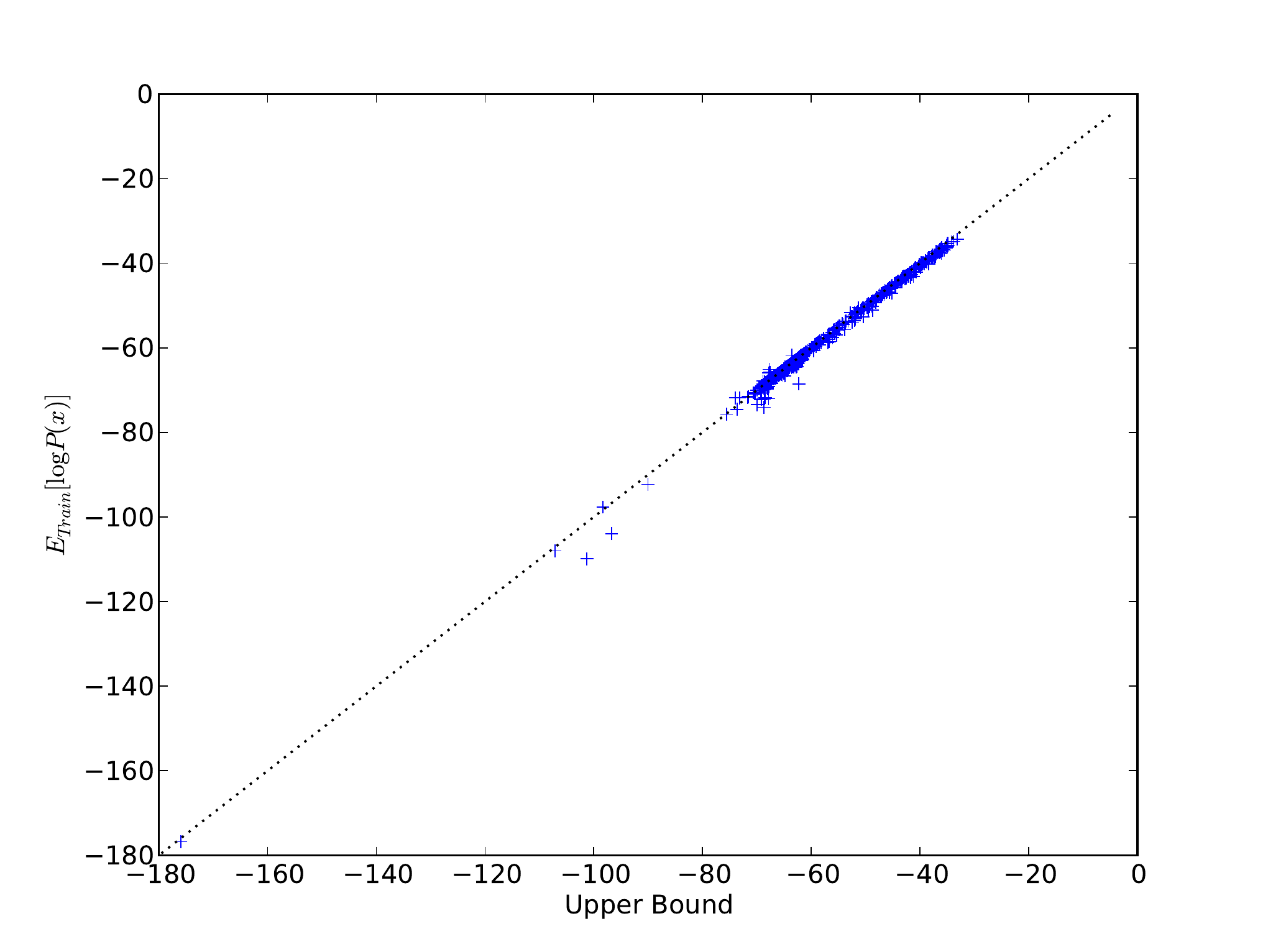}
\par\end{centering}

\caption{\label{fig:cmnist-ae3-training-upper-bound}Comparison of the BLM upper bound
on the first layer and the final log-likelihood on the \noun{Cmnist}
training dataset, for 1000 2-layer
AERIes}
\end{figure}

If the top part of the model were not powerful enough (e.g., if the
network is not deep enough), the BLM upper bound would be too
optimistic and thus significantly higher than the final log-likelihood of the
model. To further test this intuition we now compare the BLM upper bound of the
bottom layer with the log-likelihood obtained by a shallow architecture
with \emph{only one layer}; the difference would give an indication of
how much could be gained by adding top layers.
Figures~\ref{fig:tea-onerbm-training-upper-bound} and~\ref{fig:cmnist-onerbm-training-upper-bound}
compare the expected
log-likelihood\footnote{The log-likelihood reported in this specific experiment is in fact obtained with Annealed Importance Sampling (AIS).} of the training set under the 1000 RBMs previously trained with
the BLM upper bound\footnote{The BLM upper bound value given in this
particular experiment is in fact a close approximation (see Section~\ref{sec:blmapprox}).}
for a generative model using this RBM as first layer.
The results contrast with the previous
ones and confirm that final performance is below the BLM upper bound when
the model does not
have enough layers.

\begin{figure}
\begin{centering}
\includegraphics[clip,width=0.8\columnwidth]{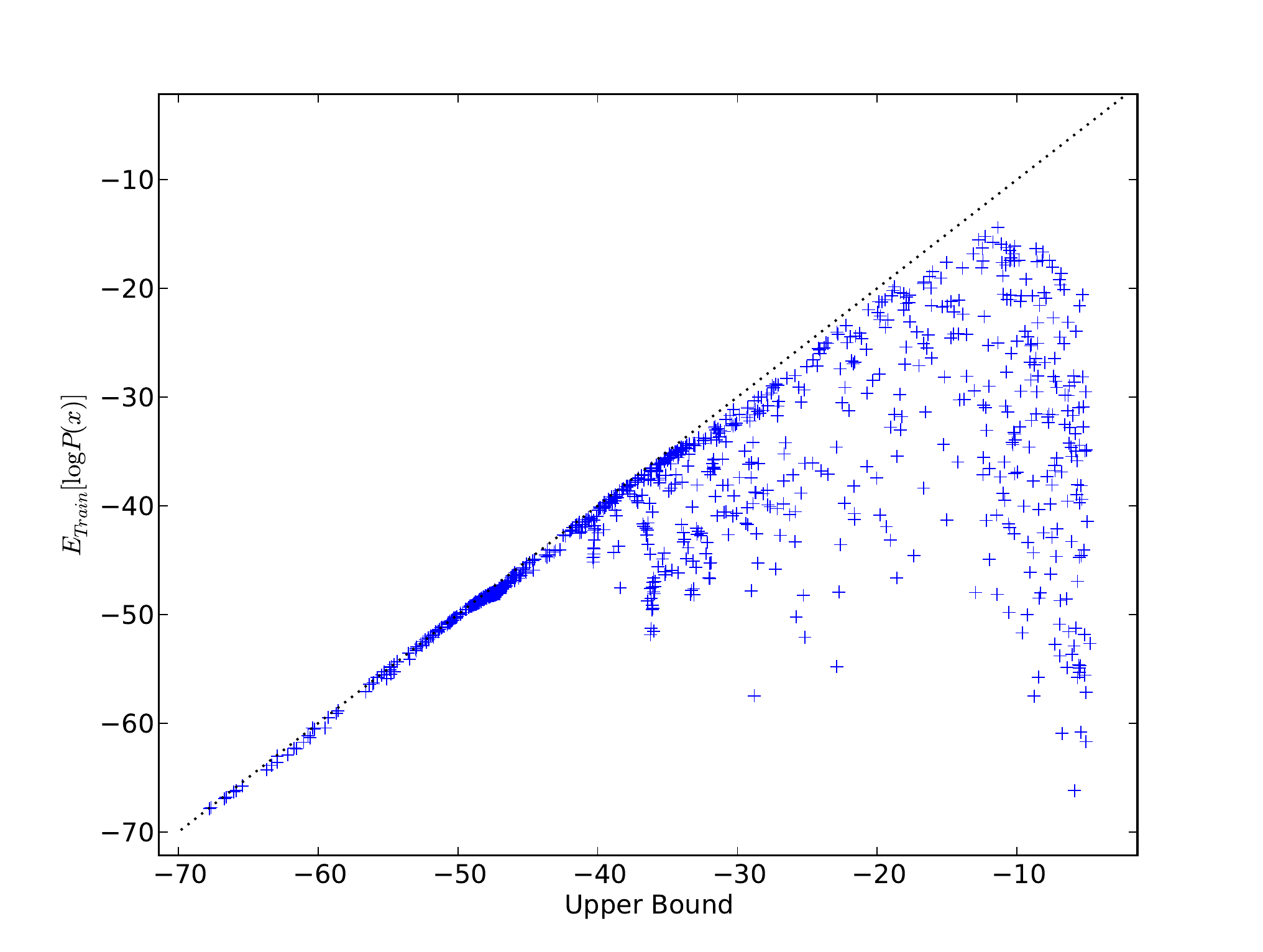}
\par\end{centering}

\caption{\label{fig:tea-onerbm-training-upper-bound}
BLM on a too shallow model:
comparison of the BLM
upper bound and the AIS log-likelihood of an RBM on the \noun{Tea}
training dataset}
\end{figure}

\begin{figure}
\begin{centering}
\includegraphics[clip,width=0.8\columnwidth]{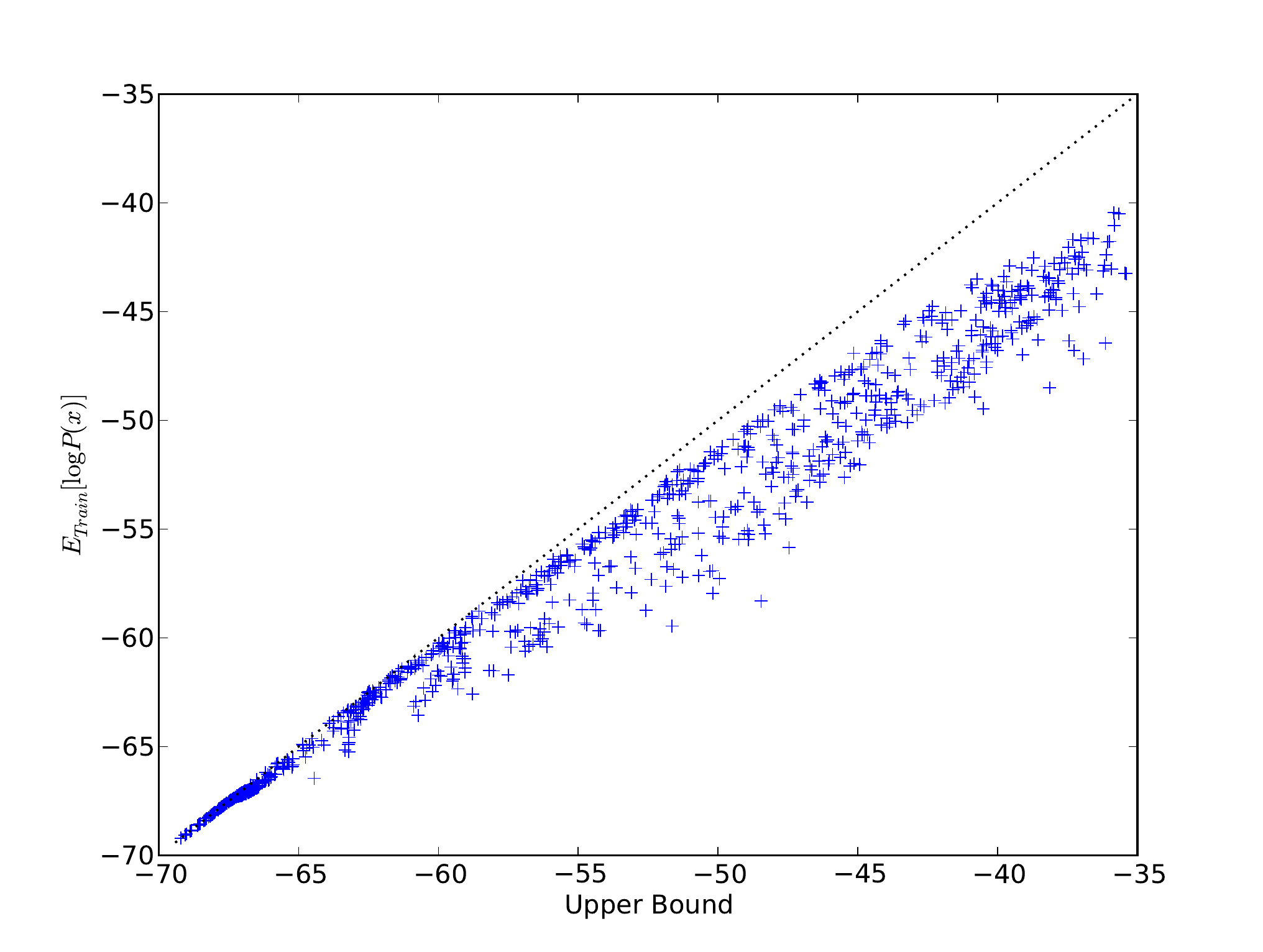}
\par\end{centering}

\caption{\label{fig:cmnist-onerbm-training-upper-bound}BLM on a too
shallow model: Comparison of the BLM
upper bound and the AIS log-likelihood of an RBM on the \noun{Cmnist}
training dataset}
\end{figure}

The alignment in Figures~\ref{fig:tea-training-upper-bound}
and~\ref{fig:cmnist-training-upper-bound} can therefore be seen as a confirmation that the \noun{Tea} and \noun{Cmnist} datasets would not benefit from a third layer.

Thus, the BLM upper bound could be used as a test for the opportunity of
adding layers to a model.

\paragraph{Training BLM upper bound vs validation log-likelihood.}

We now compare the training BLM upper bound with the log-likelihood on a
\emph{validation} set distinct from the training set:
this tests whether the BLM obtained
during training is a good indication of the final performance of a bottom
layer parameter.

\begin{figure}
\begin{centering}
\includegraphics[clip,width=0.8\columnwidth]{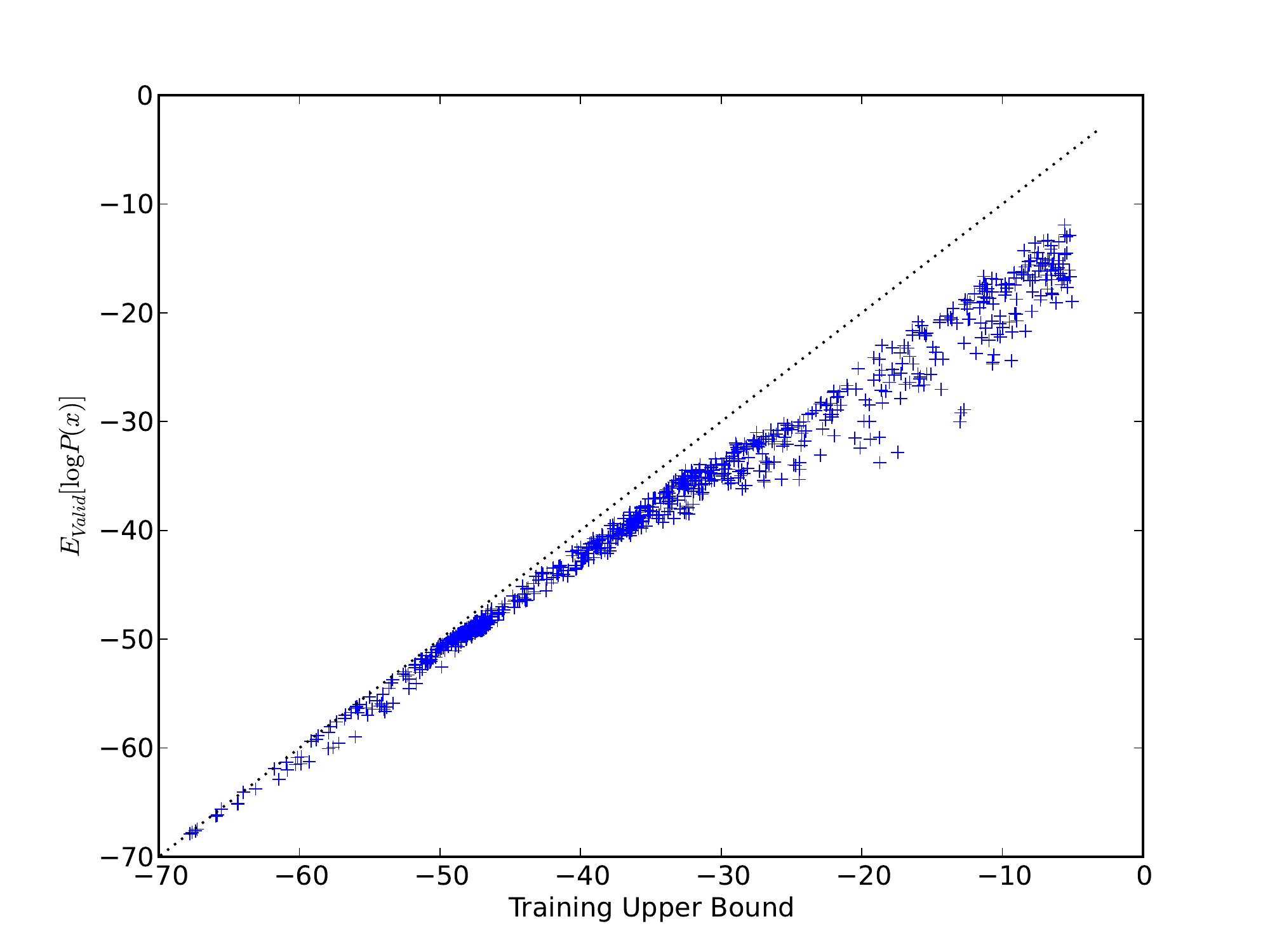}
\par\end{centering}

\caption{\label{fig:tea-rbm-training-upper-bound-vs-validationll}
Training BLM upper bound vs validation log-likelihood on the \noun{Tea}
training dataset}
\end{figure}

\begin{figure}
\begin{centering}
\includegraphics[clip,width=0.8\columnwidth]{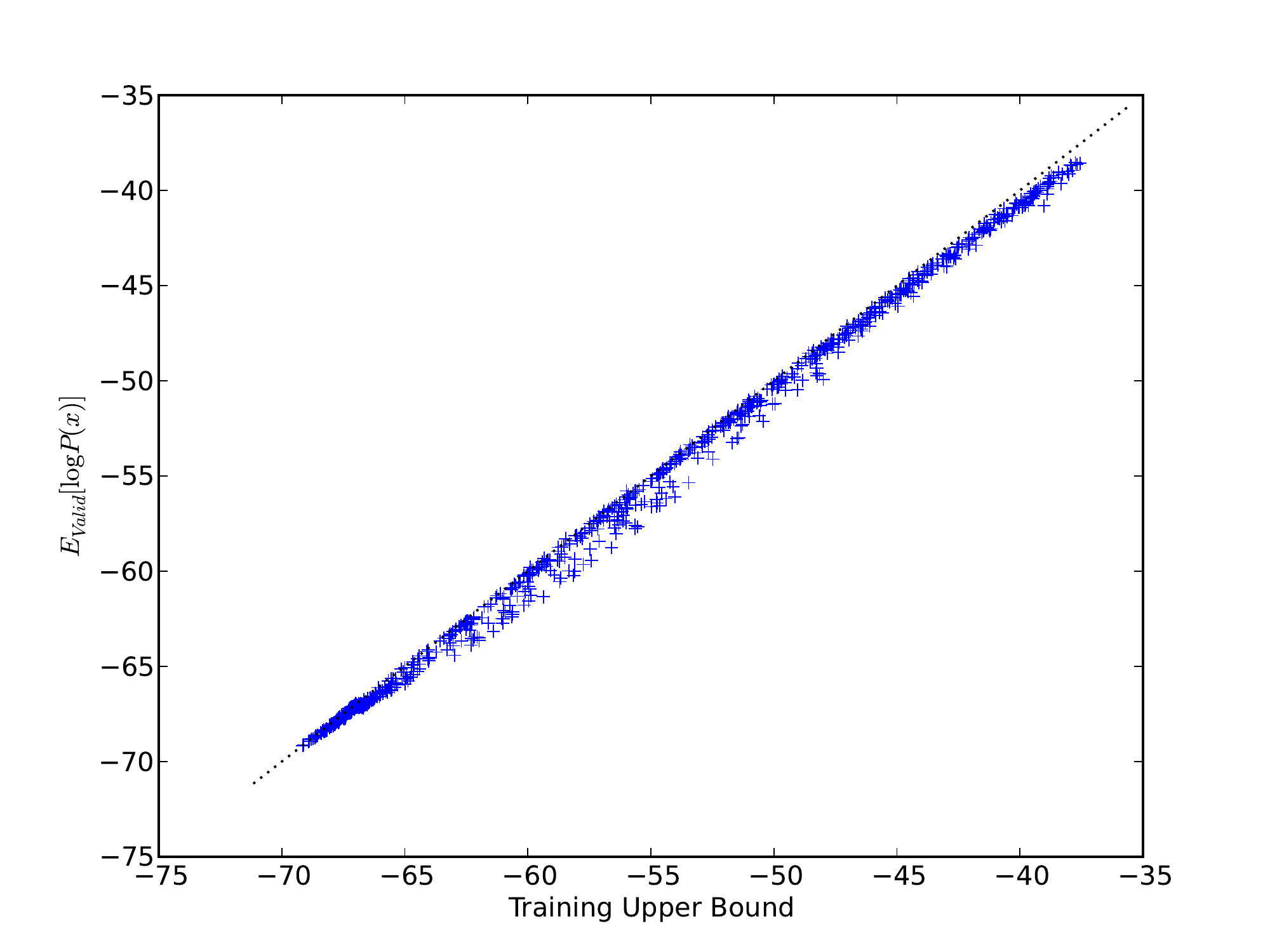}
\par\end{centering}

\caption{\label{fig:cmnist-rbm-training-upper-bound-vs-validationll}
Training BLM upper bound vs validation log-likelihood on the \noun{Cmnist}
training dataset}
\end{figure}

As discussed earlier, because the BLM makes an
assumption where there is no regularization, using the training BLM upper
bound to predict performance on a validation set could be too optimistic:
therefore we expect the validation log-likelihood to be somewhat lower
than the training BLM upper bound. (Although, paradoxically, this
can be somewhat counterbalanced by imperfect training of the upper
layers, as mentioned above.)

The results are reported in
Figures~\ref{fig:tea-rbm-training-upper-bound-vs-validationll}
and~\ref{fig:cmnist-rbm-training-upper-bound-vs-validationll} and confirm
that the training BLM is an upper bound of the validation log-likelihood.
As for regularization, we can see that on the \noun{Cmnist}
dataset where there are $4000$ samples, generalization is not very
difficult: the optimal $P(\h)$ for the
training set used by the BLM is in fact almost optimal for the
validation set too. On the
\noun{Tea} dataset, the picture is somewhat different: there is a
gap between the training upper-bound and the validation
log-likelihood. This can be attributed to the increased importance of
regularization on this dataset in which the training set contains only
$81$ samples.

Although the training BLM upper bound can therefore not be considered a good
predictor of the validation log-likelihood, it is still a monotonous
function of the validation log-likelihood: as such it can still be used
for comparing parameter settings and for
hyper-parameter selection.

\paragraph{Feeding the validation dataset to the BLM.}

Predictivity of the BLM (e.g., for hyper-parameter selection) can be improved by
feeding the \emph{validation} rather than training set to the inference
distribution and the BLM.

In the cases above we examined the predictivity of the BLM obtained
during training, on final performance on a validation dataset. We have seen
that the training BLM is an imperfect predictor of
this performance, notably because of lack of regularization in the BLM
optimistic assumption, and because we use an inference distribution $q$ maximized over the training set.

Some of these effects can easily be 
predicted by feeding the \emph{validation} set to the BLM and the
inference part of the model during hyper-parameter selection, as follows.

We call \emph{validation BLM upper bound}\NDY{It's not really an upper
bound... OK, but to be replaced if we have a better idea in the future} the BLM upper bound obtained by using the
validation dataset instead of $\D$ in \eqref{eq:defhatU}. Note that the
values $ q$ and $ \theta_I$ are still those obtained from
training on the training dataset.
This parallels the validation step for
auto-encoders, in which, of course, reconstruction performance on a
validation dataset is done by feeding this same dataset to the network.

We now compare the validation BLM upper bound to the
log-likelihood of the validation dataset, to see if
it qualifies as a reasonable proxy.

The results are reported in Figures~\ref{fig:tea-validation-upper-bound}
and~\ref{fig:cmnist-validation-upper-bound}. As predicted, the validation
BLM upper bound is a better estimator of the validation log-likelihood
(compare Figures~\ref{fig:tea-rbm-training-upper-bound-vs-validationll}
and~\ref{fig:cmnist-rbm-training-upper-bound-vs-validationll}).

We
can see that several models have a validation log-likelihood higher than
the validation BLM upper bound, which might seem paradoxical. This is
simply because the validation BLM upper bound still uses the parameters
trained on the training set and thus is not formally an upper bound.

The better overall approximation of the validation log-likelihood seems
to indicate that performing hyper-parameter
selection with the \emph{validation} BLM upper bound can better account for
generalization and regularization.

\begin{figure}
\begin{centering}
\includegraphics[clip,width=0.8\columnwidth]{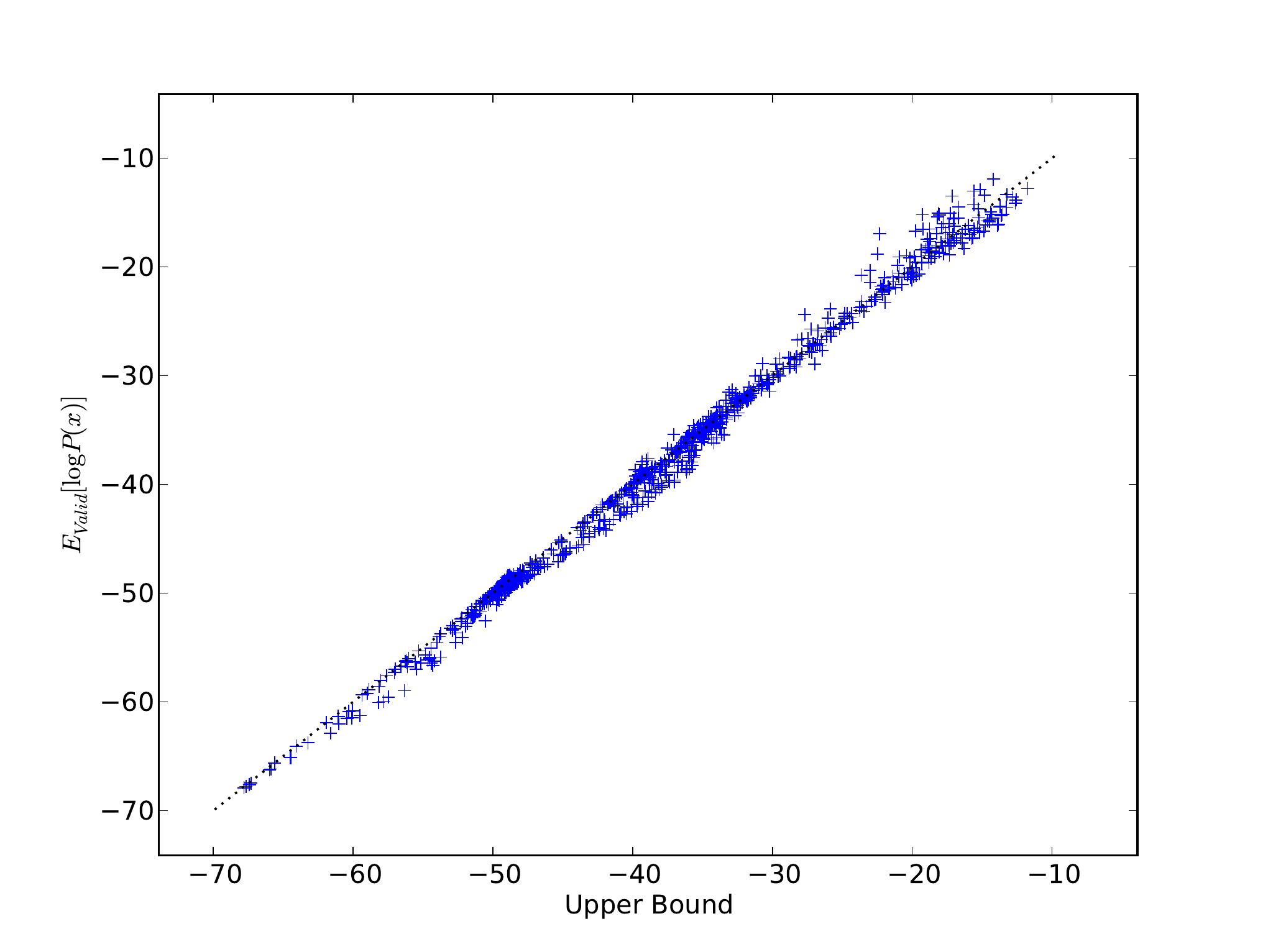}
\par\end{centering}

\caption{\label{fig:tea-validation-upper-bound}Validation upper bound vs log-likelihood on the \noun{Tea}
validation dataset}
\end{figure}

\begin{figure}
\begin{centering}
\includegraphics[clip,width=0.8\columnwidth]{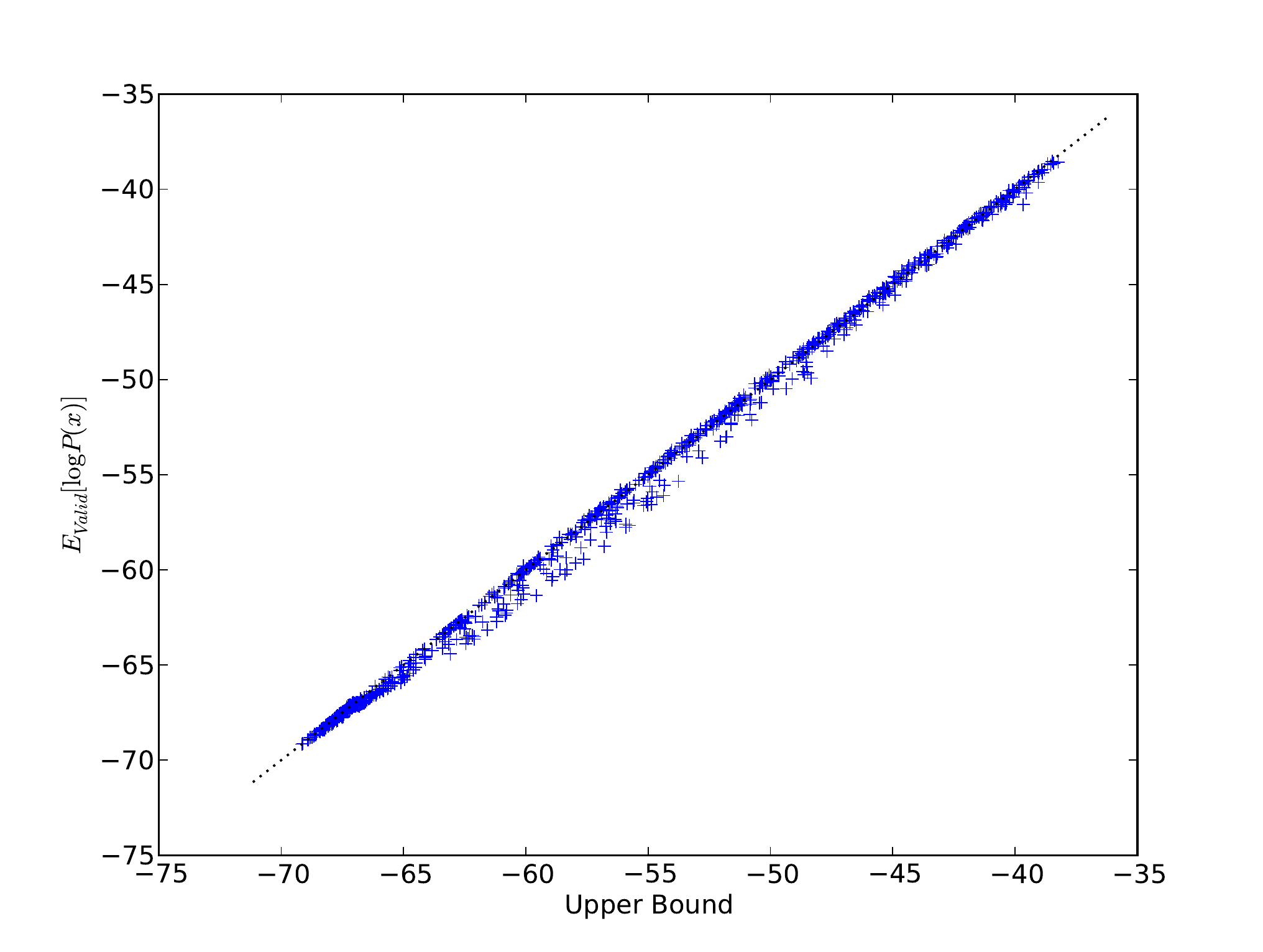}
\par\end{centering}

\caption{\label{fig:cmnist-validation-upper-bound}Validation upper bound vs log-likelihood on the \noun{Cmnist}
validation dataset}
\end{figure}

\subsubsection{Approximating the BLM for larger models}
\label{sec:blmapprox}

The experimental setting considered here was small enough to allow
for an exact computation of the various BLM bounds by summing over all
possible states of the hidden variable $\h$.
However the exact computation of the BLM upper bound using $\hat\U_{\D,{q}}(\theta_I)$ as in \eqref{eq:defhatU}
is not always possible because the number of terms in
this sum is exponential in the dimension of the hidden
layer $\mathbf{h}$.

In this situation we can use a sampling approach.
For each data sample $\tilde\x$, we can
take $K$ samples from each mode of the BLM distribution $q_\D$ (one mode for
each data sample $\mathbf{\tilde{x}}$) to obtain an approximation of the upper bound
in $\mathcal{O}(K \times N^2)$ where $N$ is the size of the
validation set. (Since the practitioner can choose the size of the
validation set which need not necessarily be as large as the training or
test sets, we do not consider the $N^2$ factor a major hurdle.)

\begin{defi}
For $\theta_I$ and ${q}$ resulting from the training of a deep generative model, let 
\begin{eqnarray}
\hat{\hat\U}_{\D,q}(\theta_I)
\deq
\mathbb{E}_{\x\sim
P_{\mathcal{D}}}\left[
\log \sum_\mathbf{\tilde{x}} \sum_{k=1}^{K} P_{\theta_I}(\x|\h) q(\h|\mathbf{\tilde{x}})P_\D(\mathbf{\tilde{x}})\right]
\label{eq:defhathatU}
\end{eqnarray}
where for each $\tilde \x$ and $k$, $\h$ is sampled from
$q(\h|\tilde \x)$\NDY{$\x$ is
not sampled, it is summed over the dataset}\NDL{True. nice catch.}.
\end{defi}

To assess the accuracy of this approximation, we take $K=1$ and compare
the values of $\hat{\hat\U}_{\D,{q}}(\theta_I)$ and
of $\hat\U_{\D,{q}}(\theta_I)$, on the \noun{Cmnist} and
\noun{Tea} training datasets.
The results are
reported in Figures~\ref{fig:tea-upper-bound-approx}
and~\ref{fig:cmnist-upper-bound-approx} for all three models (vanilla
AEs, AERIes, and SRBMs) superimposed,
showing good agreement.

\begin{figure}
\begin{centering}
\includegraphics[clip,width=0.8\columnwidth]{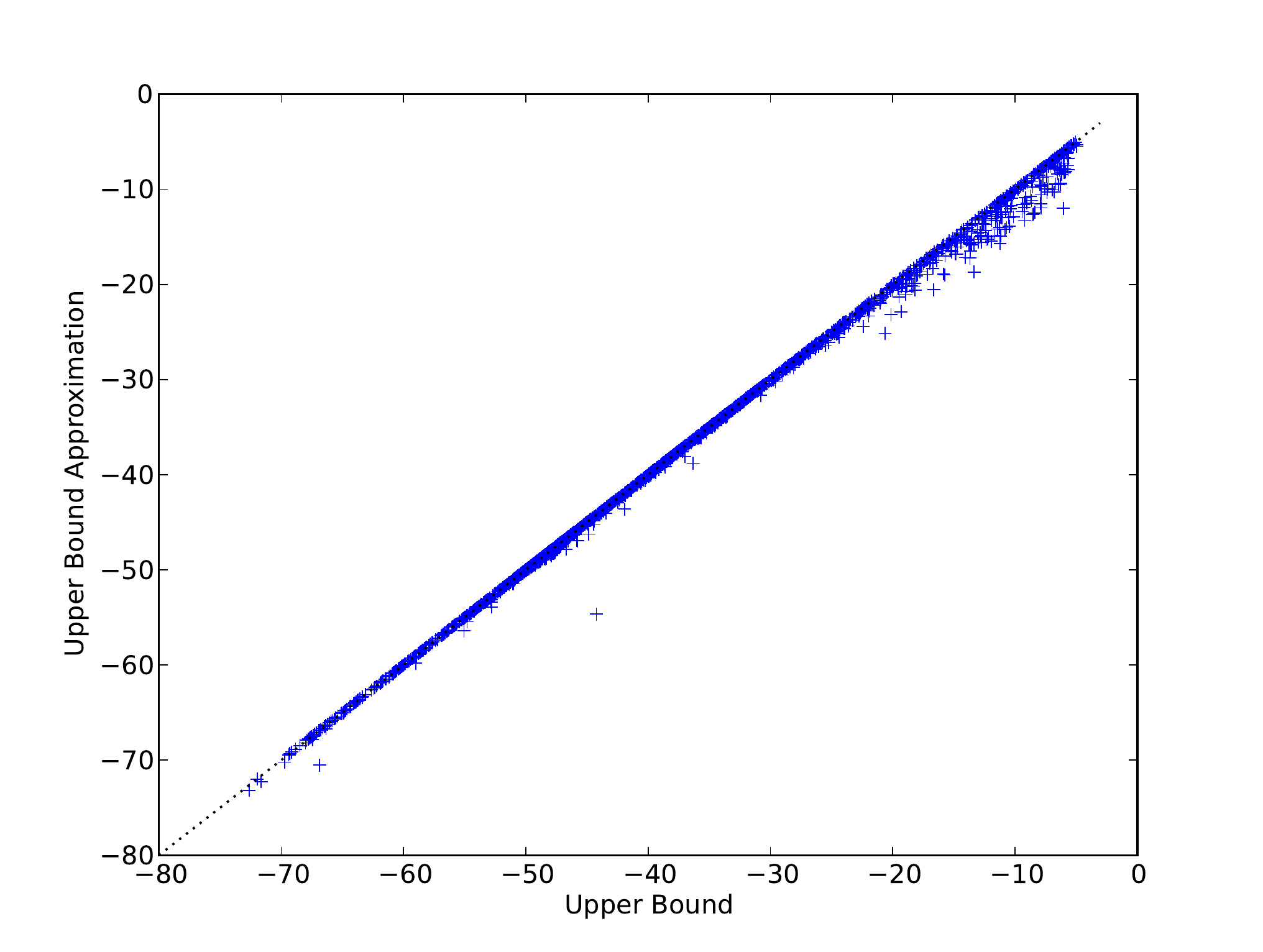}
\par\end{centering}

\caption{\label{fig:tea-upper-bound-approx}Approximation of the training BLM upper bound on the
\noun{Tea} training dataset}
\end{figure}

\begin{figure}
\begin{centering}
\includegraphics[clip,width=0.8\columnwidth]{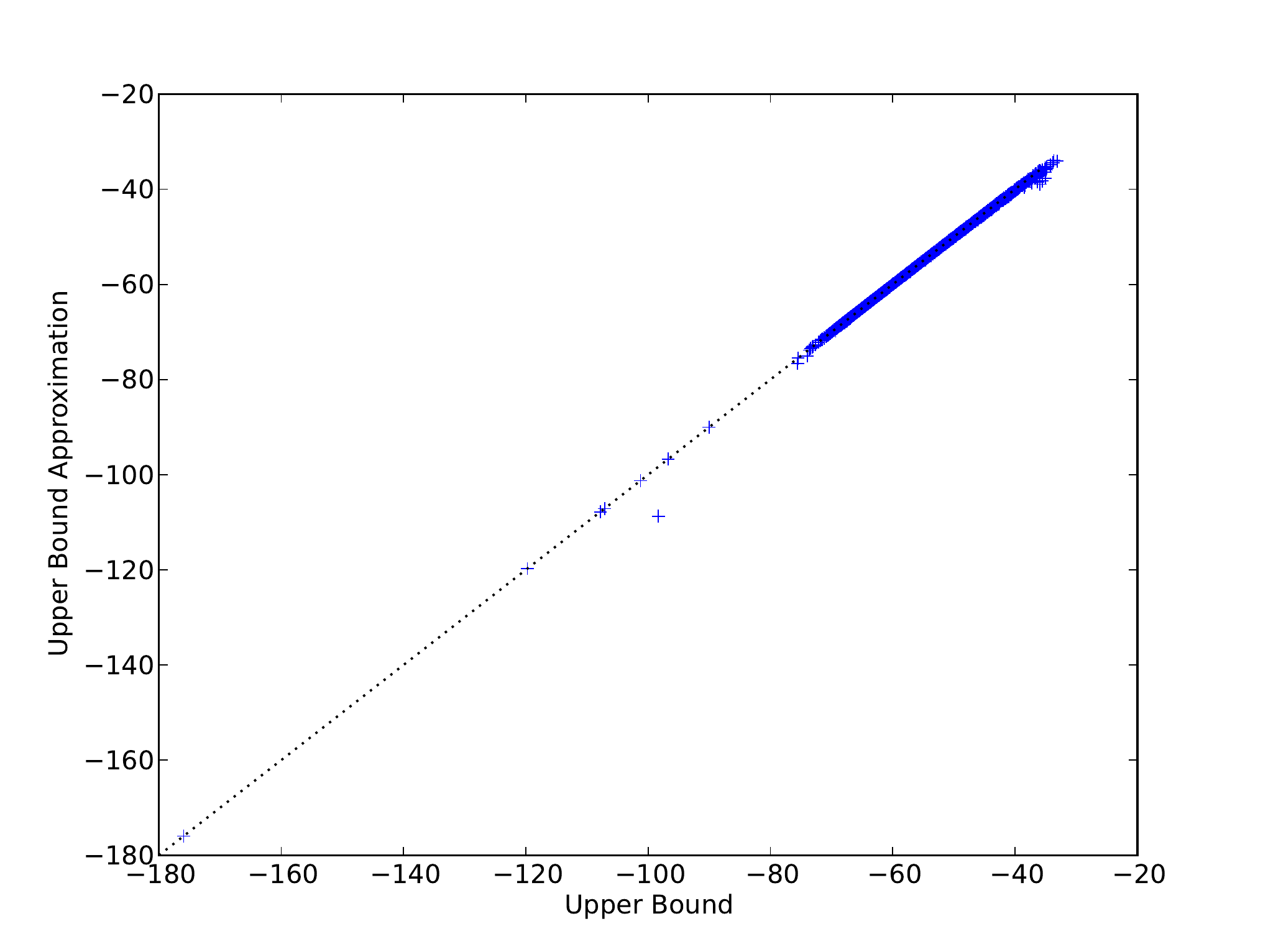}
\par\end{centering}

\caption{\label{fig:cmnist-upper-bound-approx}Approximation of the training BLM upper bound on the
the \noun{Cmnist} training dataset}
\end{figure}

\section*{Conclusions}

The new layer-wise approach we propose to train deep generative models is
based on an optimistic criterion, the BLM upper bound, in which we
suppose that learning will be successful for upper layers of the model.
Provided this optimism is justified a posteriori and a good enough model is found for
the upper layers, the resulting deep generative model is provably close
to optimal. When optimism is not justified, we provide an
explicit
bound on the loss of performance.

This provides a new justification for auto-encoder training and
fine-tuning, as the training of the lower part of a deep generative model,
optimized using a lower bound on the BLM.

This new framework for training deep generative models highlights the
importance of using richer models when performing inference, contrary to
current practice.  
This is consistent with the intuition that it is much harder to guess the
underlying structure by looking at the data, than to derive the data from
the hidden structure once it is known.

This possibility is tested empirically with auto-encoders with
rich inference (AERIes) which are
completed with a top-RBM to create deep generative models: these are then
able to outperform current state of the art (stacked RBMs) on two different deep datasets.

The BLM upper bound is also found to be a good layer-wise proxy to
evaluate the log-likelihood of future models for a given lower layer
setting, and as such is a relevant means of hyper-parameter selection.

This opens new avenues of research, for instance in the
design of algorithms to learn features in the lower part of the model, or
in the possibility to consider feature extraction as a partial deep
generative model in which the upper part of the model is temporarily
left unspecified.

\bibliographystyle{plain}
\bibliography{deeptrain}

\begin{thebibliography}{10}

\bibitem{Bengio2007a}
Y.~Bengio, P.~Lamblin, V.~Popovici, and H.~Larochelle.
\newblock Greedy layer-wise training of deep networks.
\newblock In B.~Sch\"{o}lkopf, J.~Platt, and T.~Hoffman, editors, {\em Advances
  in Neural Information Processing Systems 19}, pages 153--160. MIT Press,
  Cambridge, MA, 2007.

\bibitem{Bengio2007b}
Y.~Bengio and Y.~LeCun.
\newblock Scaling learning algorithms towards ai.
\newblock In {\em Large-Scale Kernel Machines}. MIT Press, 2007.

\bibitem{Bengio2012}
Yoshua Bengio, Aaron~C. Courville, and Pascal Vincent.
\newblock Unsupervised feature learning and deep learning: A review and new
  perspectives.
\newblock {\em CoRR}, abs/1206.5538, 2012.

\bibitem{Bengio2009}
Yoshua Bengio and Olivier Delalleau.
\newblock Justifying and generalizing contrastive divergence.
\newblock {\em Neural Computation}, 21(6):1601--1621, 2009.

\bibitem{Bergstra2011}
James Bergstra, R{\'{e}}my Bardenet, Yoshua Bengio, and Bal{\'{a}}zs
  K{\'{e}}gl.
\newblock Algorithms for hyper-parameter optimization.
\newblock In {\em Advances in Neural Information Processing Systems 23}, 2011.

\bibitem{Bergstra2012}
James Bergstra and Yoshua Bengio.
\newblock Random search for hyper-parameter optimization.
\newblock {\em Journal of Machine Learning Research}, 13:281--305, 2012.

\bibitem{Bourlard1988}
H.~Bourlard and Y.~Kamp.
\newblock Auto-association by multilayer perceptrons and singular value
  decomposition.
\newblock {\em Biological Cybernetics}, 59:291--294, 1988.

\bibitem{Buntine1991}
Wray~L. Buntine and Andreas~S. Weigend.
\newblock Bayesian back-propagation.
\newblock {\em Complex Systems}, 5:603--643, 1991.

\bibitem{CoverThomas}
Thomas~M. Cover and Joy~A. Thomas.
\newblock {\em Elements of information theory}.
\newblock Wiley-Interscience [John Wiley \& Sons], Hoboken, NJ, second edition,
  2006.

\bibitem{Dempster1977}
A.~P. Dempster, N.~M. Laird, and D.~B. Rubin.
\newblock Maximum likelihood from incomplete data via the {EM} algorithm.
\newblock {\em Journal of the Royal Statistical Society. Series B
  (Methodological)}, 39:1--38, 1977.

\bibitem{Hinton2002}
G.E. Hinton.
\newblock Training products of experts by minimizing contrastive divergence.
\newblock {\em Neural Computation}, 14:1771--1800, 2002.

\bibitem{Hinton2006}
G.E. Hinton, S.~Osindero, and Yee-Whye Teh.
\newblock A fast learning algorithm for deep belief nets.
\newblock {\em Neural Conputation}, 18:1527--1554, 2006.

\bibitem{Hinton2006a}
G.E. Hinton and R.~Salakhutdinov.
\newblock Reducing the dimensionality of data with neural networks.
\newblock {\em Science}, 313(5786):504--507, July 2006.

\bibitem{Larochelle2009}
H.~Larochelle, Y.~Bengio, J.~Louradour, and P.~Lamblin.
\newblock Exploring strategies for training deep neural networks.
\newblock {\em The Journal of Machine Learning Research}, 10:1--40, 2009.

\bibitem{Larochelle2007}
H.~Larochelle, D.~Erhan, A.~Courville, J.~Bergstra, and Y.~Bengio.
\newblock An empirical evaluation of deep architectures on problems with many
  factors of variation.
\newblock In {\em ICML '07: Proceedings of the 24th international conference on
  Machine learning}, pages 473--480, New York, NY, USA, 2007. ACM.

\bibitem{LeRoux2008}
Nicolas Le~Roux and Yoshua Bengio.
\newblock Representational power of restricted {B}oltzmann machines and deep
  belief networks.
\newblock {\em Neural Computation}, 20:1631--1649, June 2008.

\bibitem{LeCun1998a}
Y.~LeCun, L.~Bottou, Y.~Bengio, and P.~Haffner.
\newblock Gradient-based learning applied to document recognition.
\newblock {\em Proceedings of the IEEE}, 86(11):2278--2324, November 1998.

\bibitem{Neal1990}
R.~M. Neal.
\newblock Learning stochastic feedforward networks.
\newblock Technical report, Dept. of Computer Science, University of Toronto,
  1990.

\bibitem{Neal1998}
Radford~M. Neal.
\newblock Annealed importance sampling.
\newblock Technical report, University of Toronto, Department of Statistics,
  1998.

\bibitem{Rifai2012}
Salah Rifai, Yoshua Bengio, Yann Dauphin, and Pascal Vincent.
\newblock A generative process for sampling contractive auto-encoders.
\newblock In {\em International Conference on Machine Learning, ICML'12}, 06
  2012.

\bibitem{Salakhutdinov2009}
Ruslan Salakhutdinov and Geoffrey Hinton.
\newblock Deep {B}oltzmann machines.
\newblock In {\em Proceedings of the Twelfth International Conference on
  Artificial Intelligence and Statistics (AISTATS)}, volume~5, pages 448--455,
  2009.

\bibitem{Salakhutdinov2008}
Ruslan Salakhutdinov and Iain Murray.
\newblock On the quantitative analysis of deep belief networks.
\newblock In {\em Proceedings of the 25th international conference on Machine
  learning}, ICML '08, pages 872--879, New York, NY, USA, 2008. ACM.

\bibitem{Smolensky1986}
P.~Smolensky.
\newblock Information processing in dynamical systems: foundations of harmony
  theory.
\newblock In D.~Rumelhart and J.~McClelland, editors, {\em Parallel Distributed
  Processing}, volume~1, chapter~6, pages 194--281. MIT Press, Cambridge, MA,
  USA, 1986.

\bibitem{Vincent2008}
Pascal Vincent, Hugo Larochelle, Yoshua Bengio, and Pierre-Antoine Manzagol.
\newblock Extracting and composing robust features with denoising autoencoders.
\newblock In {\em Proceedings of the 25th international conference on Machine
  learning}, ICML '08, pages 1096--1103, New York, NY, USA, 2008.

\bibitem{Wu1983}
C.~F.~Jeff Wu.
\newblock On the convergence properties of the {EM} algorithm.
\newblock {\em The Annals of Statistics}, 11:95--103, 1983.

\end{thebibliography}

\end{document}